\documentclass[12pt,reqno]{amsart}

\usepackage{amsmath}
\usepackage{amsthm}
\usepackage{amssymb}
\usepackage{mathtools}
\usepackage{graphicx}
\usepackage{natbib}
\usepackage{url}
\usepackage{fullpage}

\usepackage[utf8]{inputenc}

\usepackage{amsfonts}
\usepackage{nicefrac}
\usepackage{xcolor}
\usepackage{adjustbox}

\usepackage{amsmath,amssymb,amsthm,bm}
\usepackage{amsaddr}

\theoremstyle{plain}
\newtheorem{thm}{Theorem}

\newtheorem{lem}{Lemma}
\newtheorem{prop2}{Proposition}

\theoremstyle{definition}
\newtheorem{defi}{Definition}
\newtheorem{ex}{Example}

\renewcommand{\tilde}{\widetilde}

\usepackage{color}

\allowdisplaybreaks

\title{Approximation of Permutation Invariant Polynomials by Transformers: Efficient Construction in Column-Size}

\author{Naoki Takeshita$^\dagger$, Masaaki Imaizumi$^{\dagger \ddagger}$}

\address{$^\dagger$The University of Tokyo, \\
$^\ddagger$RIKEN Center for Advanced Intelligence Project}

\thanks{Date: \today, Contact: \url{takeshita-naoki649@g.ecc.u-tokyo.ac.jp}}

\allowdisplaybreaks

\begin{document}

\maketitle

\begin{abstract}
    Transformers are a type of neural network that have demonstrated remarkable performance across various domains, particularly in natural language processing tasks. Motivated by this success, research on the theoretical understanding of transformers has garnered significant attention. A notable example is the mathematical analysis of their approximation power, which validates the empirical expressive capability of transformers. In this study, we investigate the ability of transformers to approximate column-symmetric polynomials—an extension of symmetric polynomials that take matrices as input. Consequently, we establish an explicit relationship between the size of the transformer network and its approximation capability, leveraging the parameter efficiency of transformers and their compatibility with symmetry by focusing on the algebraic properties of symmetric polynomials.
\end{abstract}

\section{Introduction}
The Transformer, a neural network extension proposed by \cite{Vas17}, has played a central role in data-driven large language models such as the Generative Pre-trained Transformer (GPT) (\cite{Bro20}) and Bidirectional Encoder Representations from Transformers (BERT) (\cite{Dev19}), becoming a primary focus in natural language processing tasks. It is also well known for its high performance in tasks beyond natural language processing, such as image processing (\cite{Dos21}), where convolutional neural networks have traditionally dominated, thereby broadening its range of applications.

The property that neural networks and Transformers can represent a wide range of functions is called the universal approximation property and has been extensively studied. While it is widely known that neural networks can approximate any continuous function to an arbitrary degree of accuracy (e.g., \cite{Cyb89, Hor91, Yar17, Lu21}), it has also been shown that Transformers exhibit similar properties. For a representative example, \cite{Yun19} demonstrated that Transformers possess the universal approximation property, allowing them to approximate permutation equivariant functions (i.e., when the columns of the input matrix are swapped, the entries of the output matrix change in the same way) with matrix inputs.

An important challenge in Transformer theory is the rigorous investigation of its approximation efficiency. For example, in \cite{Yun19}, the number of parameters required to construct the approximating Transformer increases exponentially with the size of the input matrix, highlighting significant room for improvement in approximation efficiency.
\cite{Kaj24} reveals that a Transformer with a single attention layer is a universal approximator of permutation equivariant functions. \cite{Tak23} investigates a special class of functions on an infinite-dimensional set and shows that Transformers can approximate such functions with a number of parameters that does not depend on the infinite dimensionality. Despite these results, the efficiency of Transformers in approximating a general class of functions remains an ongoing area of research.

In this paper, we focus on approximating permutation invariant polynomials and study the approximation efficiency of Transformers. Since Transformers are permutation equivariant, it is natural to consider permutation invariant functions when approximating functions with a one-dimensional output. In addition, since analytic functions can be arbitrarily well approximated by polynomials, it makes sense to restrict the target function to polynomials, which are relatively easy to approximate. We demonstrate explicit relationships between the width and depth of the Transformer network and its approximation error. We also focus on parameter efficiency: in our approximation, the number of parameters is independent of the number of columns in the input matrices. This result contrasts with conventional neural networks, where the number of parameters increases as the input dimension increases. The proof is constructive, providing explicit parameter configurations for the approximating Transformer.

To prove the result above, we extend symmetric polynomials to matrices and focus on approximating column-symmetric polynomials. While it is known that any symmetric polynomial can be expressed as a linear combination of monomial symmetric polynomials, we utilize the analogous result that any column-symmetric polynomial can be expressed as a linear combination of monomial column-symmetric polynomials. In constructing monomial column-symmetric polynomials, we focus on their algebraic properties, particularly the number of columns involved in the polynomial terms (referred to as the \textit{rank} in this paper). By discussing these polynomials inductively based on their rank, we derive the network size required to approximate monomial column-symmetric polynomials and analyze the approximation errors.

\subsection{Related Works}
As this research concerns universal approximation, we discuss several related works on this topic.

\subsubsection{Approximation Theory of Neural Networks}
The universal approximation theorem was established by works such as \cite{Cyb89} and \cite{Hor91}. These demonstrate that single-layer neural networks with a general activation function can approximate any continuous function on compact domains to an arbitrary degree of precision. However, they focus solely on the existence of neural networks that approximate continuous functions and do not specify their exact architecture.

In recent years, the empirical success of deep neural networks in tasks such as image recognition and object detection has sparked significant interest in their representational abilities. \cite{Yar17} showed that deep neural networks can approximate smooth functions with fewer parameters than shallow ones. Further developments include the results of \cite{Lu17}, which showed that ReLU FNNs with a fixed width and arbitrary depth can approximate any Lebesgue-integrable function to an arbitrary degree of precision in the sense of \(L^1\). \cite{Lu21} proved that \(C^s\)-functions can be uniformly approximated with an error that decreases at a polynomial rate with respect to the number of layers and width. In addition, \cite{Lu21} showed that the polynomial order of the uniform approximation error is optimal, except for a logarithmic factor.

As a more applied approximation theory, \cite{schmidt2017nonparametric} clarified the approximation performance of functions with composite structures and demonstrated the suitability of neural networks. \cite{petersen2018optimal} and \cite{imaizumi2018deep,imaizumi2022advantage} analyzed the approximation rate of neural networks for non-differentiable functions and showed that neural networks with more layers can achieve better approximation rates than conventional methods.
\cite{nakada2020adaptive} and \cite{chen2019efficient} demonstrated the approximation performance of neural networks adapted to manifold structures, showing that the order of approximation error can be fully described in terms of the manifold dimension. \cite{suzuki2018adaptivity} and \cite{hayakawa2020minimax} elucidated the approximation performance in the general function space such as the Besov space, demonstrating that deep neural networks can achieve optimal approximation rate even for functions that conventional methods fail to approximate optimally.

\subsubsection{Approximation Theory of Symmetric Neural Networks}
Discussions on the symmetry of neural networks have also advanced. In tasks such as image recognition, it is often desirable for the network output to be invariant to transformations such as parallel shifts. Neural networks that are inherently symmetric can be advantageous for this reason. Additionally, imposing symmetry can reduce the number of parameters, which is particularly valuable, especially given that recent models have a significant number of parameters.  
\cite{Zah17} considered neural networks defined on sets. Since sets do not take the order of their elements into account, this can be regarded as a type of symmetric neural network in the paper. Research on symmetric neural networks has also progressed. \cite{Yar18} showed that any permutation invariant function \( f: \mathbb{R}^{d \times n} \to \mathbb{R} \) with \( d, n \in \mathbb{N} \) can be uniformly approximated to arbitrary precision on any compact set using a two-layer neural network that is permutation invariant with respect to its input columns. Additionally, \cite{Mar19} extended these results to the general case, demonstrating that functions invariant under specific permutations are universal approximators. Furthermore, \cite{San19} proved the permutation equivariant case and showed that imposing symmetry on ReLU FNNs reduces the number of parameters compared to the case without symmetry. These approaches provide an effective framework for approximating symmetric functions by leveraging inherent symmetries.

\subsubsection{Universal Approximation of Transformers}
The universal approximation theorem for Transformers, proved by \cite{Yun19}, states that any continuous permutation equivariant function on $[0,1]^{d \times n}$ can be approximated to an arbitrary precision by Transformers. Additionally, it demonstrates that if positional encoding (a method for embedding positional information into input data) is employed, the same result holds even when the target function is not permutation equivariant. \cite{Tak23} showed that specific shift equivariant functions (i.e., functions equivariant to column shifts) can be approximated by a one-layer Transformer with positional encoding. Later, \cite{Kaj24} showed that by directly using the softmax function—in contrast to \cite{Yun19}, which used the hardmax function—a Transformer with a single-head attention layer serves as a universal approximator for permutation equivariant continuous functions. Moreover, several other universal approximation theorems have been established, considering different cases. \cite{Zah20} demonstrated that Transformers with sparse attention layers are universal approximators while reducing computational complexity in the attention layers. \cite{Yun20} investigated a more general case of universal approximation with sparse attention layers. \cite{Kra22} examined universal approximation under constraints, where the outputs of both the target function and the approximating Transformer lie within a specific convex set.

\subsubsection{Approximation Efficiency of Neural Networks and Transformers}

There are studies related to the expressivity of neural networks and Transformers beyond universal approximation.
Let $d$ and $n$ be the input dimension and sequence length of a Transformer (i.e., the number of input tokens).
\cite{Bho20} proved that certain matrices cannot be expressed as the output of the softmax function in the attention layer of Transformers when $d < n$.
\cite{Lik21} demonstrated that the number of columns required to approximate sparse matrices is significantly lower than the total number of columns.

Another topic related to efficiency is memorization capacity, which focuses on fitting $N$ input-output pairs.
\cite{Par21} showed that ReLU FNNs with $\tilde{O}(N^{2/3})$ parameters can memorize $N$ pairs, where $\tilde{O}(\cdot)$ is Landau's Big-O notation which omits constants and logarithmic factors.
\cite{Var22} improved the rate to $\tilde{O} (\sqrt{N})$, which is optimal when ignoring logarithmic factors, according to \cite{Gol93}.
The case for Transformers was shown in \cite{Kim23}, showing that Transformers with $\tilde{O} (d + n + \sqrt{dN})$ parameters can memorize $N$ input-output mappings, where the inputs belong to $\mathbb{R}^{d\times n}$.

\subsection{Notation}
\label{sec1.3}
We denote matrices by uppercase boldface letters, such as $\boldsymbol{A} \in \mathbb{R}^{d \times n}$, and vectors by lowercase boldface letters, such as $\boldsymbol{x} = (x_1, \dots, x_d)^\top \in \mathbb{R}^d$. Vector and matrix addition is defined element-wise. Let $\boldsymbol{O}_{d\times n} \in \mathbb{R}^{d\times n}$ denote the zero matrix and $\boldsymbol{1}_{d\times n}$ denote the $d\times n$ matrix where all entries are equal to 1. Block matrices are represented as  
$
\boldsymbol{A} = \begin{pmatrix} \boldsymbol{A}_{11} & \boldsymbol{A}_{12} \\ \boldsymbol{A}_{21} & \boldsymbol{A}_{22} \end{pmatrix}
$  
and $[\boldsymbol{A}]_i$ represents the $i$-th column of $\boldsymbol{A}$.  
For any positive integer $n$ and for any vectors $\boldsymbol{x} = (x_1, \dots, x_n) \in \mathbb{R}^n$ and $\boldsymbol{\alpha} = (\alpha_1, \dots, \alpha_n) \in \mathbb{N}^n$, we define  
$
\boldsymbol{x}^{\boldsymbol{\alpha}} \coloneqq x_1^{\alpha_1} x_2^{\alpha_2} \cdots x_n^{\alpha_n},
$  
and  
$
|\boldsymbol{x}| \coloneqq \lVert \boldsymbol{x} \rVert_1 = |x_1| + |x_2| + \dots + |x_n|,
$  
which represents the degree of the monomial $\boldsymbol{x}^{\boldsymbol{\alpha}}$.
We compare vectors $\boldsymbol{x} = (x_1, \dots, x_n)$ and $\boldsymbol{y} = (y_1, \dots, y_n)$ based on lexicographical order: that is, $\boldsymbol{x} < \boldsymbol{y}$ if $x_1 < y_1$; otherwise, if $x_1 = y_1$, the comparison is determined by the values of $x_2$ and $y_2$; if $x_2 = y_2$, the process continues with $x_3$ and $y_3$, and so on.
Let $S_n$ be the set of all permutations of $(1,2,\dots,n)$.

\subsection{Paper Organization}
Section \ref{sec2} introduces the basic concepts used in this study, followed by the main theorem in Section \ref{sec3}. The proof of the main theorem is constructed in several steps. First, we approximate column-wise monomials in Section \ref{sec4}. Next, in Section \ref{sec4.4}, we construct rank-1 monomial column-symmetric polynomials by summing the column-wise monomials. Then, in Section \ref{sec5}, we inductively approximate rank-$r\ (\geq 2)$ monomial column-symmetric polynomials based on their ranks. Finally, in Section \ref{sec6}, we complete the approximation of column-symmetric polynomials to prove the main theorem. In Section \ref{sec7}, we provide further discussion on this study. Additionally, some basic properties are proved in the appendix.

\section{Preliminaries}
\label{sec2}
In this section, we define important concepts, such as symmetric polynomials and the Transformer, to state the main theorem in Section \ref{sec3}.

\subsection{Neural Networks and Transformers}
First, we introduce feed-forward neural networks (FNNs), a typical type of neural network. Here, we consider the Rectified Linear Unit (ReLU) activation function, which is defined as $\mathrm{ReLU}: \mathbb{R}^d \to \mathbb{R}^d: (x_1, \dots, x_d)^\top \mapsto (\max(0, x_1), \dots, \max(0, x_d))^\top$. FNNs take vectors as inputs and return real numbers as outputs. Strictly speaking, a ReLU FNN is defined as follows.

\begin{defi}[ReLU feed-forward neural network]
    Fix a number $L \in \mathbb{N}_+$ and $d_0,...,d_{L+1} \in \mathbb{N}_+$, where $d_{L+1} = 1$.
    For the input $\boldsymbol{x}_0 \in \mathbb{R}^{d_0}$, a ReLU FNN is a function which returns $\mathrm{NN}(\boldsymbol{x}_0) = \boldsymbol{y}_L \in \mathbb{R}$, which is described as follows:
    a sequence $\{\boldsymbol{y}_i = (y_{i1, \dots, id_i})^\top \in \mathbb{R}^{d_{i+1}}\}_{i=0}^L$  is defined by the recursive manner for $i = 0, 1, \dots, L$:
    \begin{equation*}
        \boldsymbol{y}_{i} = \boldsymbol{W}_i \boldsymbol{x}_i + \boldsymbol{b}_i,
    \end{equation*}
    where $\boldsymbol{x}_{i} \in \mathbb{R}^{d_i}$ is defined as 
    \begin{equation*}
        \boldsymbol{x}_{i} = \mathrm{ReLU} (\boldsymbol{y}_{i-1})\quad \text{for }i=0,1,...,L-1.
    \end{equation*}
    Here, $\boldsymbol{W}_i \in \mathbb{R}^{d_{i+1}\times d_{i}}$ and $\boldsymbol{b}_i \in \mathbb{R}^{d_{i+1}}$ are parameter matrices and bias vectors, respectively.
\end{defi}

$N = \max (d_1, \dots, d_L)$ is referred to as the \textit{width} of $\mathrm{NN}(\boldsymbol{x}_0)$, and $L$ as its \textit{depth}. The vectors $\boldsymbol{x}_1, \dots, \boldsymbol{x}_L$ are referred to as the hidden layers of $\mathrm{NN}(\boldsymbol{x}_0)$. The vector $\boldsymbol{x}_{i+1}$ is obtained by applying an affine transformation followed by the ReLU activation function to $\boldsymbol{x}_i$ for $i = 0,1,\dots,L-1$. Note that the activation function is not applied when obtaining the final output $\boldsymbol{y}_L$ from the last hidden layer $\boldsymbol{x}_L$. An example of a ReLU FNN is illustrated in Figure \ref{fig1}.
\begin{figure}[htbp]
    \centering
    \includegraphics[width=0.9\linewidth]{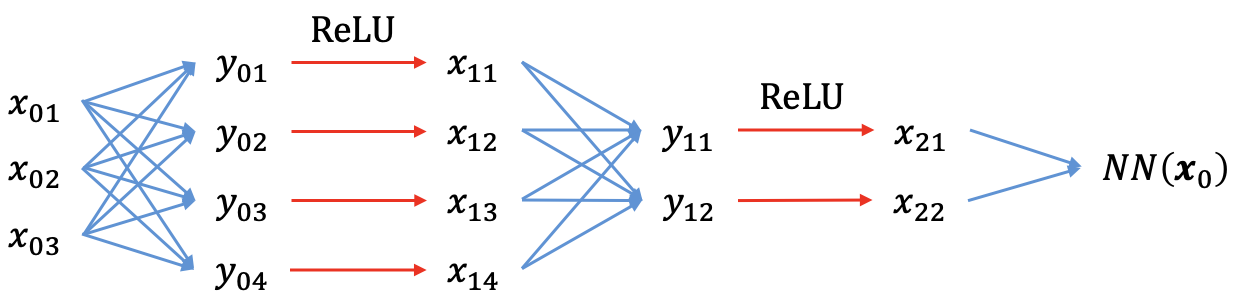}
    \caption{An Example of a ReLU FNN with width $N=4$ and depth $L=2$.}
    \label{fig1}
\end{figure}

Next, we introduce the Transformer, an extension of the FNN. A Transformer receives matrices and returns matrices of the same size. It is constructed by repeatedly combining \textit{Transformer Blocks}, each of which consists of two layers: an attention layer and a feed-forward layer. More precisely, we define the Transformer as follows:

\begin{defi}[Transformer]
    Fix $h,d,n,m,k \in \mathbb{N}_+$.
    For any matrix $\boldsymbol{X} \in \mathbb{R}^{d\times n}$, we define an \textit{attention layer} $\mathrm{Attn}: \mathbb{R}^{d\times n} \rightarrow \mathbb{R}^{d\times n}$ with width $d$ and $h$ heads as 
    \begin{align*}
        \mathrm{Attn} (\boldsymbol{X}) \coloneqq & \boldsymbol{X} + 
        \sum_{i=1}^h \boldsymbol{W}_O^i \boldsymbol{W}_V^i \boldsymbol{X} \cdot \mathrm{softmax} \left[ (\boldsymbol{W}_K^i \boldsymbol{X})^\top \boldsymbol{W}_Q^i \boldsymbol{X} \right],
    \end{align*}
    where $\boldsymbol{W}_O^i \in \mathbb{R}^{d\times m}, \boldsymbol{W}_V^i, \boldsymbol{W}_K^i, \boldsymbol{W}_Q^i \in \mathbb{R}^{m\times d}$ are parameter matrices for $i=1,...,h$.
    Here, the softmax function $\mathrm{softmax}:\mathbb{R}^{n \times n} \to \mathbb{R}^{n \times n} $ is applied column-wise: i.e. for a matrix $\boldsymbol{A} \in \mathbb{R}^{n \times n}$, we define
    \begin{equation*}
        \mathrm{softmax}(\boldsymbol{A})= \mathrm{softmax} \left(  
        \begin{bmatrix}
            a_{11} & a_{12} & \cdots & a_{1n} \\
            a_{21} & a_{22} & \cdots & a_{2n} \\
            \vdots & \vdots & \ddots & \vdots \\
            a_{n1} & a_{n2} & \cdots & a_{nn} 
        \end{bmatrix}
        \right) \coloneqq
        \begin{bmatrix}
            e^{-a_{11}} / s_1 & e^{-a_{12}} / s_2 & \cdots & e^{-a_{1n}} / s_n \\
            e^{-a_{21}} / s_1 & e^{-a_{22}} / s_2 & \cdots & e^{-a_{2n}} / s_n \\
            \vdots & \vdots & \ddots & \vdots \\
            e^{-a_{n1}} / s_1 & e^{-a_{n2}} / s_2 & \cdots & e^{-a_{nn}} / s_n
        \end{bmatrix} ,
    \end{equation*}
    where $s_i = e^{-a_{1i}} + e^{-a_{2i}} + \dots + e^{-a_{di}}$.
    Next, a \textit{feed-forward layer} $ \mathrm{FF}: \mathbb{R}^{d\times n} \rightarrow \mathbb{R}^{d\times n}$ in a matrix form is defined 
    \begin{align*}
        \mathrm{FF} (\boldsymbol{X}) \coloneqq & \boldsymbol{X} + \boldsymbol{W}_2 \cdot \mathrm{ReLU} (\boldsymbol{W}_1 \boldsymbol{X} + \boldsymbol{b}_1 \boldsymbol{1}_n^\top) +  \boldsymbol{b}_2 \boldsymbol{1}_n^\top,
    \end{align*}
    where $\boldsymbol{W}_1^\top, \boldsymbol{W}_2 \in \mathbb{R}^{d\times r}$ are parameter matrices and $\boldsymbol{b}_1 \in \mathbb{R}^{r}, \boldsymbol{b}_2 \in \mathbb{R}^{d})$ are the bias vectors.
    Finally, a \textit{Transformer block} $ \mathrm{TB}: \mathbb{R}^{d\times n} \rightarrow \mathbb{R}^{d\times n}$ is defined as 
    \begin{align*}
        \mathrm{TB} (\boldsymbol{X}) \coloneqq & \mathrm{FF}(\mathrm{Attn}(\boldsymbol{X})).
    \end{align*}
    A function $\mathrm{TF}$ obtained by composing $\mathrm{TB}(\cdot)$ $k$ times is referred to as a \textit{Transformer network} of width $d$ and depth $k$.
\end{defi}

This definition omits layer normalization unlike \cite{Vas17}, for brevity. However, We denote that this does not affect the expressivity of the Transformer.
Figure \ref{fig:tf} illustrates the architecture of Transformers.

\begin{figure}[htbp]
    \centering
    \includegraphics[width=0.25\linewidth]{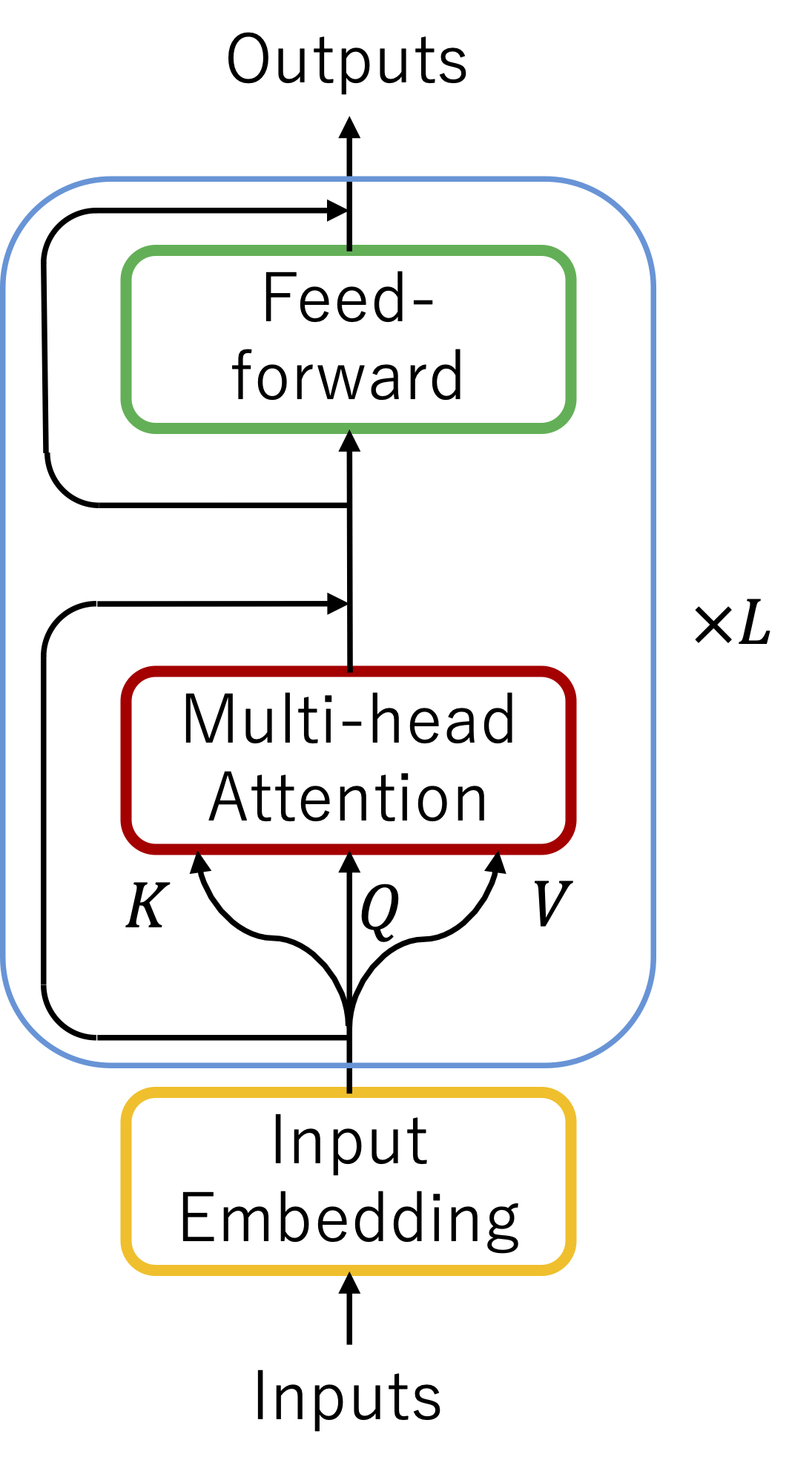}
    \caption{Architecture of Transformers}
    \label{fig:tf}
\end{figure}

\subsection{Symmetric Polynomials}
\label{sec2.2}

In this section, we define symmetric polynomials where the input is a vector, and then consider the general case where the input is a matrix.

\begin{defi}
    A function $f: \mathbb{R}^n \to \mathbb{R}^m$ is \textit{permutation invariant} if $f(x_{\sigma(1)}, x_{\sigma(2)}, \dots, x_{\sigma(n)}) = f(x_{1}, x_{2}, \dots, x_{n})$ holds for any permutation $(\sigma(1), \sigma(2), \dots, \sigma(n)) \in S_n$. In particular, if a polynomial with degree $s$ (the degree of the polynomial refers to the highest degree of the terms. e.g., the polynomial $x_1^2x_2 + x_3$ is a degree-3 polynomial.) is permutation invariant, we call it a \textit{symmetric degree-$s$ polynomial}. For example, $x_1^2 x_2 x_3 + x_2^2 x_1 x_3 + x_3^2 x_1 x_2 + 4(x_1 + x_2 + x_3)$ is a symmetric polynomial over $\boldsymbol{x} = (x_1, x_2, x_3)$. However, $f(x_1, x_2, x_3) = x_1^2 + x_2 + x_3$ is not a symmetric polynomial as $f(x_2, x_3, x_1) = x_2^2 + x_3 + x_1 \neq f(x_1, x_2, x_3)$.
\end{defi}

Next, we consider the matrix input case and define column-symmetric polynomials over matrices. 

\begin{defi}
    A function $f(\boldsymbol{X}): \mathbb{R}^{d\times n} \rightarrow \mathbb{R}^{d' \times n'}$ with some $d' \in \mathbb{N}_+$ is \textit{column permutation invariant} if $f(\boldsymbol{x}_{\sigma(1)}, \boldsymbol{x}_{\sigma(2)}, \dots, \boldsymbol{x}_{\sigma(n)}) = f(\boldsymbol{x}_1, \boldsymbol{x}_2, \dots, \boldsymbol{x}_n)$ holds for any permutation $\sigma \in S_n$. In particular, we call a column permutation invariant degree-$s$ polynomial as a \textit{column-symmetric degree-$s$ polynomial}. 
\end{defi}

\section{Universal Approximation of Column-Symmetric Polynomials}
\label{sec3}
The main statement of this study demonstrates the relationship between the width and depth of the Transformer network and its approximation error. Since the approximating Transformer has only a single attention head, and its width and depth are independent of the number of input columns (denoted as \( n \)), the number of parameters is also independent of the number of input columns.

\begin{thm}
\label{th1}
    Let $f(\boldsymbol{X})$ be an arbitrary degree-$s$ column-symmetric polynomial over $[0,1]^{d\times n}$ with positive coefficients, satisfying $\lVert f\rVert_{L^{\infty}} \leq 1$: i.e. $\max_{\boldsymbol{X}\in [0,1]^{d\times n}} |f(\boldsymbol{X})| \leq 1$. Then, for any $N, L \in \mathbb{N}_{+}$, there exists a 1-head Transformer network: $\mathrm{TF}$ with width at most $12\cdot (2d)^s N$, and depth at most $2sL + 3s$, which satisfies
    \begin{equation*}
        \max_{\boldsymbol{X} \in [0,1]^{d\times n}} |f(\boldsymbol{X}) - \mathrm{TF}(\boldsymbol{X})| < 8^s \cdot N^{-L},
    \end{equation*}
    which has only a single attention head.
\end{thm}

Since the coefficients in \( f(\boldsymbol{X}) \) are positive, \( \lVert f \rVert_{L^{\infty}} \), the maximum value of \( f \) in \( [0,1]^{d\times n} \), is equal to \( f(\boldsymbol{1}_{d\times n}) \), which is the sum of all coefficients appearing in \( f(\boldsymbol{X}) \). Hence, as long as the sum of the coefficients does not exceed an absolute constant, the approximation error remains independent of the number of rows \( d \) and columns \( n \) of the input matrix.

\subsection{Examples}
Here, we demonstrate some examples to make our main statement more familiar.

\begin{ex}
    Let $f_1(\boldsymbol{X})$ be the polynomial consisting of all terms of degree at most $s$, with the coefficient of every term being equal to 1. For example, for the case when $d=n=s=2$, $f_1(\boldsymbol{X})$ is equivalent to 
    \begin{equation*}
    \begin{split}
        & x_{11} + x_{12} + x_{21} + x_{22} + x_{11}^2 + x_{12}^2 + x_{21}^2 + x_{22}^2 \\
        &+  x_{11}x_{12} + x_{11}x_{21} + x_{11}x_{22} + x_{12}x_{21} + x_{12}x_{22} + x_{21}x_{22}.
    \end{split}
    \end{equation*}
    As the number of terms (terms can be written in the format of $x_{11}^{p_{11}} \dots x_{dn}^{p_{dn}}$) which appear in $f_1(\boldsymbol{X})$ is at most the number of sets of integers $(p_{11}, \dots, p_{dn})$ which satisfy 
    \begin{equation*}
        p_{11} + \dots + p_{dn} \leq s,\quad p_{11}, \dots, p_{dn} \geq 0,
    \end{equation*}
    it follows from Lemma \ref{lm2} that the value of this equation is at most
    \begin{equation*}
        \begin{pmatrix} dn + s \\ s \end{pmatrix} = \frac{dn+s}{s} \cdot \frac{dn+s-1}{s-1} \dots \frac{dn+1}{1} \leq (dn+1)^s,
    \end{equation*}
    implying $|f_1(\boldsymbol{1}_{d\times n})| \leq (dn+1)^s$. Thus, there exists a Transformer network with width at most $12\cdot (2d)^s N = 16N$, and depth at most $2sL + 3s = 4L+6$, which approximates $f_1(\boldsymbol{X})$ as
    \begin{align*}
        \max_{\boldsymbol{X} \in [0,1]^{d\times n}} |f_1(\boldsymbol{X}) - \mathrm{TF}(\boldsymbol{X})| < (8(dn+1))^s \cdot N^{-L}= 1600 N^{-L}.
    \end{align*}
\end{ex}

\begin{ex}
    We consider the case when $d = 2, n = 3, s = 4$ and 
    \begin{align*}
        f_2(\boldsymbol{X}) &= \frac{1}{9} (x_{11}^2 (x_{12}x_{22} + x_{13}x_{23}) + x_{12}^2 (x_{11}x_{21} + x_{13}x_{23}) \\
        &\quad + x_{13}^2 (x_{11}x_{21} + x_{12}x_{22}) + x_{11} + x_{12} + x_{13}).
    \end{align*}
    Since substituting $\boldsymbol{X}$ with $\boldsymbol{1}_{2\times 3}$ yields $f_2(\boldsymbol{X}) = 1$, there exists a Transformer network with width at most $12\cdot 4^4 N = 3072 N$, and with depth at most $2sL + 3s = 8L + 12$, which approximates $f_1(\boldsymbol{X})$ with an error of 
    \begin{align*}
        \max_{\boldsymbol{X} \in [0,1]^{d\times n}} |f_2(\boldsymbol{X}) - \mathrm{TF}(\boldsymbol{X})| < 8^s \cdot N^{-L}= 4096N^{-L}.
    \end{align*}
    By considering the average of all the terms in $f(\boldsymbol{X})$, the approximation error does not depend on the rows $d$ and columns $n$ of the input matrix.
\end{ex}

\section{Proof Outline}

We present an outline of the proof of the main theorem.
In preparation, we define the notation of an approximation error with a sign.
For a function $f$ and a Transformer $\mathcal{T}$, we denote 
\begin{align*}
    \mathcal{E}_{\mathcal{T}} (f(x)) \coloneqq \mathcal{T}(x)- f(x)
\end{align*}
as the signed approximation error of $f$ by $\mathcal{T}$ at point $x$.
Note that this definition can be negative.
$|\mathcal{E}_{\mathcal{T}}(f(x))|$ the approximation error without a sign.

\subsection{Monomial Column-symmetric Polynomials}
We define a certain class of symmetric polynomials called monomial column-symmetric polynomials. These polynomials play a crucial role in our proof, as we approximate column-symmetric polynomials by taking a weighted sum of monomial column-symmetric polynomials.

In the case where \( d = 1 \), monomial column-symmetric polynomials coincide with monomial symmetric polynomials, which are defined as follows:
\begin{defi}[Monomial symmetric polynomials]
    A monomial symmetric polynomial is a symmetric polynomial which can be written as the sum of permutations of a single term. since 
    and $x_1^2 x_2 x_3 + x_2^2 x_1 x_3 + x_3^2 x_1 x_2$ is a monomial symmetric polynomial over $x_1, x_2, x_3$ as $x_1^2 x_2 x_3, x_2^2 x_1 x_3, x_3^2 x_1 x_2$ are all permutations of $x_1^2 x_2 x_3$. Similarly, $x_1 + x_2$ is a monomial symmetric polynomial over $x_1, x_2$.
\end{defi}

Next, we consider the general case. For \( d \geq 2 \), we define a more complex polynomial by summing the column-permutations of a specific monomial. We provide its rigorous definition as follows:
\begin{defi}[Monomial column-symmetric polynomials]
    Fix $r \in \mathbb{N}$ and $\boldsymbol{p}_1, \dots, \boldsymbol{p}_r \in \mathbb{N}^{d}$.
    We define a \textit{rank-$r$ monomial column-symmetric polynomials} $m_{\boldsymbol{p}_1, \dots, \boldsymbol{p}_r} (\boldsymbol{X})$
    over $\boldsymbol{X} = (\boldsymbol{x}_1, \dots, \boldsymbol{x}_n) \in \mathbb{R}^{d\times n}$ as 
    \begin{equation*}
    m_{\boldsymbol{p}_1, \dots, \boldsymbol{p}_r} (\boldsymbol{X}) \coloneqq
    \begin{cases}
        \displaystyle  \sum_{\sigma \in S_n} \frac{1}{(n-r)!} \boldsymbol{x}_{\sigma(1)}^{\boldsymbol{p}_1} \dots \boldsymbol{x}_{\sigma(r)}^{\boldsymbol{p}_r} \quad &\text{if}\ r \leq n,\\
        0\quad &\text{if}\ r > n.
    \end{cases}
    \end{equation*}
\end{defi}

In the $d=1$ case, the rank of monomial column-symmetric polynomials corresponds to the degree of the polynomial, as column-symmetric polynomials are equivalent to symmetric polynomials in this case. For the case of $d \geq 2$, the rank corresponds to the number of columns that each term spans.

We note that there are $(n-r)!$ permutations in $S_n$ where $(\sigma(1), \dots, \sigma(r))$ are identical. Hence, $\frac{1}{(n-r)!} \boldsymbol{x}_{\sigma(1)}^{\boldsymbol{p}_1} \dots \boldsymbol{x}_{\sigma(r)}^{\boldsymbol{p}_r}$ is equivalent to the sum of $\boldsymbol{x}_{\sigma(1)}^{\boldsymbol{p}_1} \dots \boldsymbol{x}_{\sigma(r)}^{\boldsymbol{p}_r}$, where $(\sigma(1), \dots, \sigma(r))$ are distinct. In addition, the coefficients of each term may not necessarily be equal to $1$, as terms that become identical through permutations can be counted multiple times (e.g. $x_2 x_3 x_1$ and $x_1 x_2 x_3$ are permutations of the same monomial). For a monomial column-symmetric polynomial that includes the term $\boldsymbol{x}_1^{\boldsymbol{p}_1} \boldsymbol{x}_2^{\boldsymbol{p}_2} \dots \boldsymbol{x}_n^{\boldsymbol{p}_n}$, the corresponding coefficient is $i_1! i_2!\dots i_m!$, where the tuple $(\boldsymbol{p}_1, \boldsymbol{p}_2, \dots, \boldsymbol{p}_n)$ consists of $i_1$ occurrences of '$\boldsymbol{p}_{j_1}$', $i_2$ occurrences of '$\boldsymbol{p}_{j_2}$', $\dots$, and $i_m$ occurrences of '$\boldsymbol{p}_{j_m}$', with $\boldsymbol{p}_{j_1}, \boldsymbol{p}_{j_2}, \dots, \boldsymbol{p}_{j_m}$ being distinct. For example, when $d=1$ and $n=5$, the tuple $(\boldsymbol{p}_1, \boldsymbol{p}_2, \boldsymbol{p}_3, \boldsymbol{p}_4, \boldsymbol{p}_5) = (1,1,1,2,2)$ contains three '1's and two '2's, and hence $i_1 ! i_2! = 3! \cdot 2! = 12$.

We give several examples of the monomial column-symmetric polynomial.
\begin{ex}
    For the case when $d = 2, n = 3$, the rank-$2$ monomial column-symmetric polynomial $m_{(1,0), (1,0)} (\boldsymbol{X})$ is given by
    \begin{equation*}
    \begin{split}
        m_{(1,0), (1,0)} (\boldsymbol{X}) = & x_{11}x_{12} + x_{12}x_{13}  + x_{12}x_{11} + x_{12}x_{13} + x_{13}x_{11} + x_{13}x_{12}\\
        = & 2(x_{11}x_{12} + x_{11}x_{13} + x_{12}x_{13}).
    \end{split}
    \end{equation*}
    In this case, the coefficients of the terms are all equal to 2. This is because the tuple $((1,0), (1,0))$ has two '$(1,0)$'s.
\end{ex}

\begin{ex}
    For the case when $d = 2, n = 3$ and the rank-$2$ monomial column-symmetric polynomial $m_{(1,1), (1,0)} (\boldsymbol{X})$ is given by
    \begin{align*}
        &m_{(1,1), (1,0)} (\boldsymbol{X}) \\
        & = x_{11}x_{21}x_{12} + x_{11}x_{21}x_{13} + x_{12}x_{22}x_{11} + x_{12}x_{22}x_{13} + x_{13}x_{23}x_{11} + x_{13}x_{23}x_{12}.
    \end{align*}
    The following Figure \ref{fig2} is a corresponding illustration.
\end{ex}

\begin{figure}[htbp]
    \centering
    \includegraphics[width=0.95\linewidth]{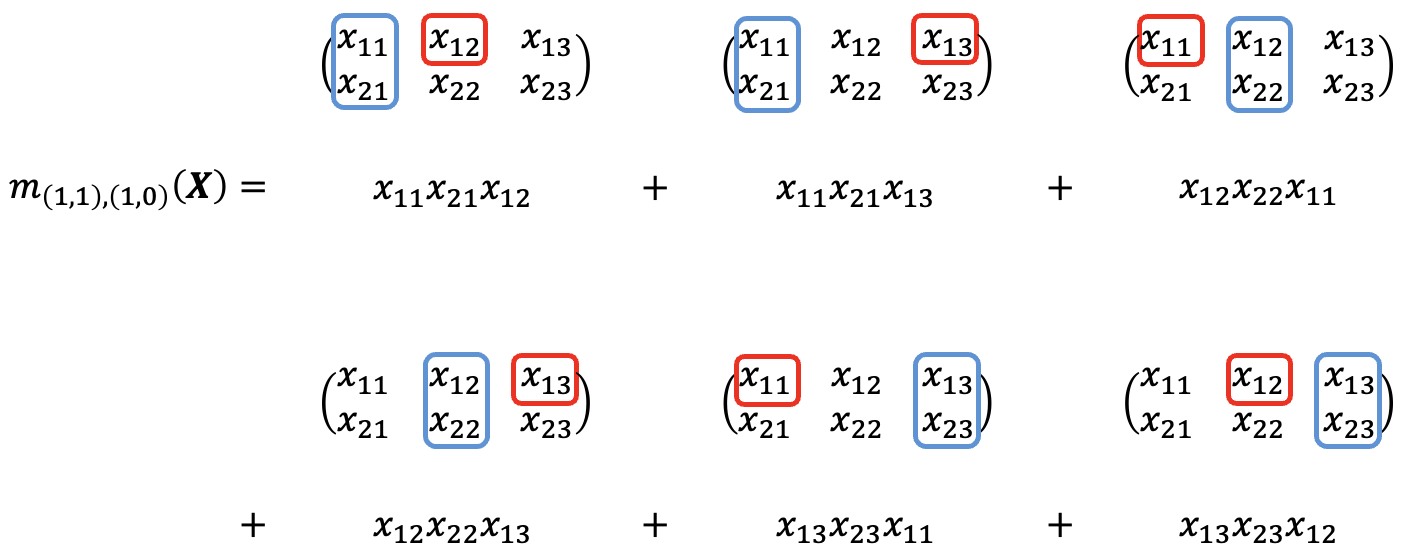}
    \caption{The illustration of $m_{(1,1),(1,0)}(\boldsymbol{X})$}
    \label{fig2}
\end{figure}

\subsection{Summary of the Proof of Theorem \ref{th1}}
From Section \ref{sec4} to Section \ref{sec6} below, we construct the approximation of $f(\boldsymbol{X})$ in multiple steps, as illustrated in Table \ref{tab1}. 

\begin{table}[htbp]
    \centering
    \begin{tabular}{c|c|c}
        Target of Approximation & Analysis & Layer used \\ \hline \hline
        $\phi(x,y) = xy$ & Section \ref{sec4.1} & Feed-forward \\ \hline
        $\psi(\boldsymbol{x}_1) = x_{11}^{p_{11}} \dots x_{d1}^{p_{d1}}$ & Section \ref{sec4.2}, \ref{sec4.3} & Feed-forward \\ \hline
        \begin{tabular}{c}
        Rank-$1$ monomial \\column-symmetric polynomial
        \end{tabular}
         & Section \ref{sec4.4} & Single-head Attention \\ \hline
        \begin{tabular}{c}
        Rank-$r$ monomial \\column-symmetric polynomial
        \end{tabular}
        & Section \ref{sec5} & Feed-forward\\ \hline
        $f(\boldsymbol{X})$ & Section \ref{sec6} & Feed-forward \\ 
    \end{tabular}
    \caption{Target functions and sections where its approximations are constructed}
    \label{tab1}
\end{table}

First, we approximate multiplication, specifically $\phi(x,y) = xy$ for $x,y \in [0,1]$. By repeatedly applying this approximation, we can approximate column-wise monomials, such as $\psi (\boldsymbol{x}_1) = x_{11}^{p_{11}} \dots x_{d1}^{p_{d1}}$. Using this method, we prove Proposition \ref{prop0}, which is demonstrated in Section \ref{sec4}.

\begin{prop2}
\label{prop0}
    There exists a feed-forward network (i.e. a Transformer network consisted only of feed-forward layers) $\mathrm{TF}_0$, whose width is at most $12sd^sN$ and depth at most $(s-1)(L+1)$, which approximates all monomials over $\boldsymbol{x}_1$, i.e. 
    \begin{equation*}
        \mathrm{TF}_0
        \begin{pmatrix}
            x_{11} \\ \vdots \\ x_{1d} \\ 0 \\ 0 \\ \vdots \\ 0
        \end{pmatrix}
        \sim 
        \begin{bmatrix}
            x_{11} \\ \vdots \\ x_{1d} \\ x_{11}^2\\ x_{11}x_{21}\\ \vdots \\ x_{d1}^s
        \end{bmatrix} .
    \end{equation*}
\end{prop2}

Second, we take the column-wise sum of these approximations of monomials, obtaining approximations of rank-$1$ monomial column-symmetric polynomials. The following proposition is proved in Section \ref{sec4.4}.

\begin{prop2}
\label{prop2}
    There exists a single-attention Transformer network $\mathrm{TF}_1$, whose width is at most $12sd^sN$ and depth at most $(s-1)(L+1) + 1$, which satisfies
    \begin{equation*}
        \left[ \mathrm{TF_1}
        \begin{pmatrix}
            x_{11} & x_{12} & \cdots & x_{1n} & 0\\
            \vdots & \vdots & \ddots & \vdots & 0\\
            x_{d1} & x_{d2} &\cdots & x_{dn} & 0\\
            0 & 0 & \cdots & 0 & 0\\
            0 & 0 & \cdots & 0 & 0\\
            \vdots & \vdots & \ddots & \vdots & \vdots\\
            0 & 0 & \cdots & 0 & 0\\
        \end{pmatrix}
        \right]_{n+1} = 
        \begin{bmatrix}
            m_{(1,0,\dots,0)}(\boldsymbol{X}) + \mathcal{E}_{\mathrm{TF}_1} (m_{(1,0,\dots,0)} (\boldsymbol{X}))\\
            \vdots \\
            m_{(0,\dots,0,1)}(\boldsymbol{X}) + \mathcal{E}_{\mathrm{TF}_1} (m_{(0,\dots,0,1)} (\boldsymbol{X}))\\
            m_{(2,0,\dots,0)}(\boldsymbol{X}) + \mathcal{E}_{\mathrm{TF}_1} (m_{(2,0,\dots,0)} (\boldsymbol{X}))\\
            m_{(1,1,0,\dots,0)}(\boldsymbol{X}) + \mathcal{E}_{\mathrm{TF}_1} (m_{(1,1,0\dots,0)} (\boldsymbol{X}))\\
            \vdots\\
            m_{(0,\dots,0,s)}(\boldsymbol{X}) + \mathcal{E}_{\mathrm{TF}_1} (m_{(0,\dots,0,s)} (\boldsymbol{X}))\\
        \end{bmatrix} ,
    \end{equation*}
    where the approximation errors $\mathcal{E}_{\mathrm{TF}_1} (m_{\boldsymbol{p}} (\boldsymbol{X}))$ satisfy
    \begin{equation*}
        |\mathcal{E}_{\mathrm{TF}_1} (m_{\boldsymbol{p}} (\boldsymbol{X}))| \leq n(|\boldsymbol{p}| - 1) N^{-L}. \quad (\boldsymbol{p} \in \mathbb{N}^d,\ 1 \leq |\boldsymbol{p}| \leq s).
    \end{equation*}
\end{prop2}

Next, we construct monomial column-symmetric polynomials of higher rank by induction on the rank. This enables us to obtain the following proposition, which is proved in Section \ref{sec5}.

\begin{prop2}
\label{prop3}
    There exists a Transformer network $\mathrm{TF}_2$ with width at most $12\cdot (2d)^s N$ and depth at most $(s-1)(L+2)$, such that 
    \begin{align*}
        &\mathrm{TF}_2
        \begin{pmatrix}
            m_{(1,0,\dots,0)} + \mathcal{E}_{\mathrm{TF}_1} (m_{(0,0,\dots,1)}) (\boldsymbol{X})\\
            \vdots\\
            m_{(0,\dots,0,1)} + \mathcal{E}_{\mathrm{TF}_1} (m_{(0,0,\dots,1)}) (\boldsymbol{X})\\
            m_{(2,0,\dots,0)} + \mathcal{E}_{\mathrm{TF}_1} (m_{(2,0,\dots,0)}) (\boldsymbol{X})\\
            \vdots\\
            m_{(0,\dots,0,s)} + \mathcal{E}_{\mathrm{TF}_1} (m_{(2,0,\dots,0)}) (\boldsymbol{X})\\
            0\\ \vdots \\ 0
        \end{pmatrix}\\
        &=
        \begin{bmatrix}
            m_{(1,0,\dots,0)} (\boldsymbol{X}) + \mathcal{E}_{\mathrm{TF}_2} (m_{(1,0,\dots,0)}) (\boldsymbol{X})\\
            \vdots\\
            m_{(0,\dots,0,1)} + \mathcal{E}_{\mathrm{TF}_2} (m_{(0,0,\dots,1)}) (\boldsymbol{X})\\
            m_{(2,0,\dots,0)} + \mathcal{E}_{\mathrm{TF}_2} (m_{(2,0,\dots,0)}) (\boldsymbol{X})\\
            \vdots\\
            m_{(0,\dots,0,s)} + \mathcal{E}_{\mathrm{TF}_2} (m_{(0,\dots,0,s)}) (\boldsymbol{X})\\
            m_{(1,0,\dots,0), (1,0,\dots,0)} + \mathcal{E}_{\mathrm{TF}_2} (m_{(1,0,\dots,0), (1,0,\dots,0)}) (\boldsymbol{X})\\
            \vdots \\
            m_{(0,\dots,0,1), \dots, (0,\dots,0,1)} + \mathcal{E}_{\mathrm{TF}_2} (m_{(0,\dots,0,1), \dots, (0,\dots,0,1)}) (\boldsymbol{X})
        \end{bmatrix} ,
    \end{align*}
    where
    \begin{equation*}
        |\mathcal{E}_{\mathrm{TF}_2} (m_{\boldsymbol{p}_1, \dots, \boldsymbol{p}_r}) (\boldsymbol{X})| \leq (P(n+r-1, r) \cdot (|\boldsymbol{p}_1| + 1) \dots (|\boldsymbol{p}_r| + 1) - n^r) N^{-L}.
    \end{equation*}
\end{prop2}

Finally, we approximate the column-symmetric polynomial \( f(\boldsymbol{X}) \) by taking a weighted sum of monomial column-symmetric polynomials. The column-wise summation is performed by a single-head attention layer, while the other processes are conducted by the feed-forward layer.

\section{Proof of Proposition \ref{prop0}}
\label{sec4}

\subsection{Approximating Products}
\label{sec4.1}
Here, we approximate the function $\phi(x,y) = xy\ (x,y\in [0,1])$ by ReLU FNNs by the method used by \cite{Yar17} and \cite{Lu21}. First, we approximate the function $x \mapsto x^2$ for $x\in [0,1]$. Let $T_1(x)\ (x\in [0,1])$ be 
\begin{equation*}
    T_1(x) \coloneqq 
    \begin{cases}
        2x\quad \text{if } x \in [0, 0.5] \\ 2(1-x)\quad \text{if } x \in [0.5, 1]
    \end{cases}
\end{equation*}
and $T_{i+1} \coloneqq T_i\circ T_1$. Then, $T_i$ becomes the sawtooth function illustrated in Figure \ref{fig3}. 
\begin{figure}[htbp]
    \centering
    \includegraphics[width=0.6\linewidth]{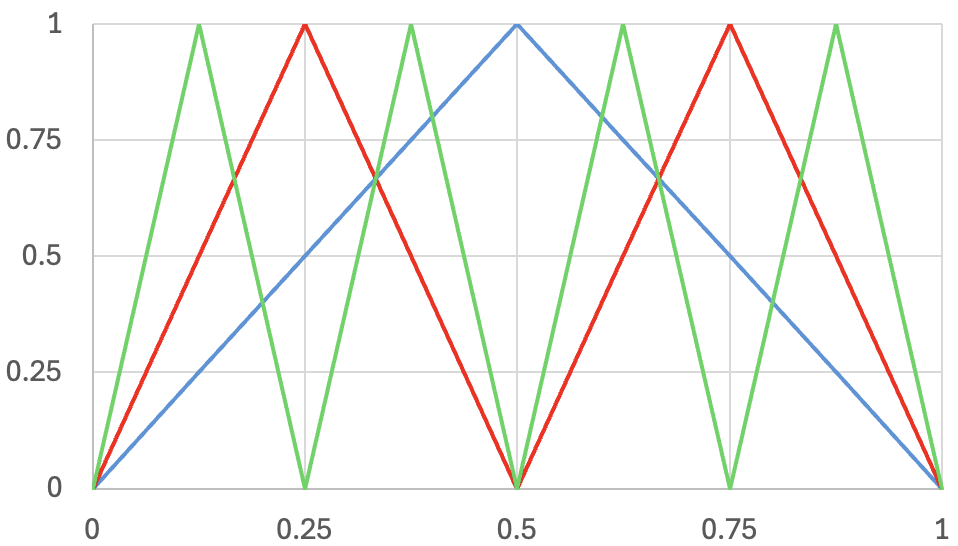}
    \caption{$T_1, T_2$ and $T_3$ are illustrated in blue, red and green respectively.}
    \label{fig3}
\end{figure}

Now, we define 
\[
\displaystyle \tilde{f}_k(x) \coloneqq x - \sum_{i=1}^k \frac{T_i(x)}{4^i}.
\]
Then, as illustrated in Figure \ref{fig4}, $\tilde{f}_{k}(x)$ approximates the target function $x \mapsto x^2$. 
\begin{figure}[htbp]
    \centering
    \includegraphics[width=0.55\linewidth]{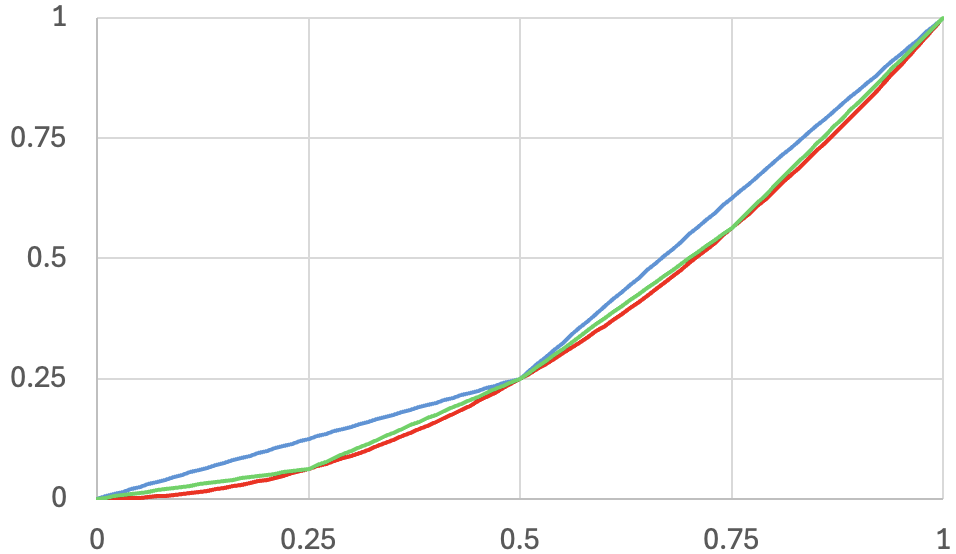}
    \caption{$\tilde{f}_1(x)$ (in blue) and $\tilde{f}_2(x)$ (in green) approximating the target function $x \mapsto x^2$ (in red)}
    \label{fig4}
\end{figure}
As a result, the following statement holds, which will be proved in \ref{sec9.2}.

\begin{lem}
\label{lm7}
    The equality 
    \begin{equation*}
        \displaystyle \tilde{f}_k(x) - x^2 = - \left( x - \frac{i}{2^k} \right) \left( x - \frac{i+1}{2^k} \right) 
    \end{equation*}
    holds for any $x \in [0,1], k \geq 1$ and $\displaystyle \frac{i}{2^k} \leq x \leq \frac{i+1}{2^k}$, where $i \in \{0,1,\dots,2^k-1 \}$. In particular, $\displaystyle 0 \leq \tilde{f}_k(x) - x^2 \leq \frac{1}{4^{k+1}}$ holds for any $x \in [0,1]$.
\end{lem}

Let \( k \) be the integer satisfying \( 2^k \leq N < 2^{k+1} \). Then, from Lemma \ref{lm4}, it is easy to verify that \( T_i(x) \) is piecewise linear in the interval \( \frac{j}{2^i} \leq x \leq \frac{j+1}{2^i} \) for \( j = 0,1,\dots,2^i-1 \). Hence, \( T_i(x) \) can be constructed by a ReLU FNN with width \( 2^i \) and depth $1$. 
Note that the composition of a depth-\( L_1 \) ReLU FNN, \( \phi_1 \), and a depth-\( L_2 \) ReLU FNN, \( \phi_2: \mathbb{R} \rightarrow \mathbb{R} \), can be achieved with a depth-\( L_1 + L_2 \) ReLU FNN, since obtaining the output of \( \phi_1 \) from its last hidden layer is merely an affine transformation, which can be combined with the first affine transformation of \( \phi_2 \). 
As a result, \( \tilde{f}_{Lk}(x) \) can be approximated by the ReLU FNN illustrated in Figure \ref{fig5}, which has a width of at most \( 2 + \dots + 2^k + 1 = 2^{k+1} - 1 < 2N \) and a depth of \( L \).

\begin{figure}[htbp]
    \centering
    \includegraphics[width=0.9\linewidth]{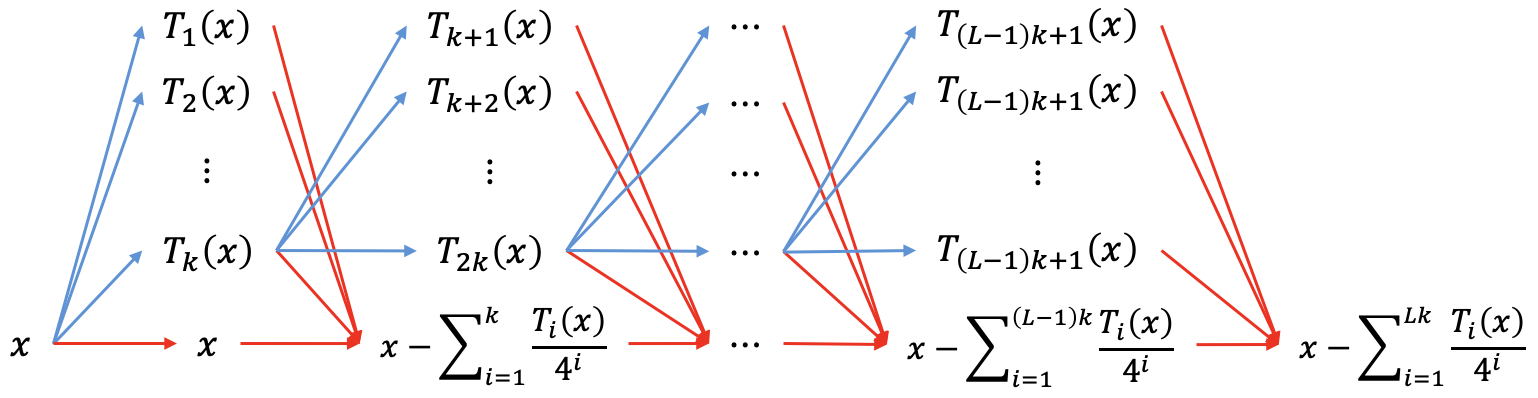}
    \caption{The illustration of $T_{Lk}$, where the blue and red lines show $T_1, \dots, T_k$ and affine transformations respectively.}
    \label{fig5}
\end{figure}

By considering that $xy = 2\left(\frac{x+y}{2}\right)^2 - \frac{1}{2}x^2 - \frac{1}{2}y^2$, the function $\tilde{g}_{Lk}(x,y) \coloneqq 2\tilde{f}_{Lk}(\frac{x+y}{2}) - \frac{1}{2} \tilde{f}_{Lk}(x) - \frac{1}{2} \tilde{f}_{Lk}(y)\ (x,y\in [0,1])$ is an approximation of $\phi (x,y) = xy$, which can be approximated by a ReLU FNN with width at most $6N$, and depth at most $L$, as illustrated in Figure \ref{fig6}. Combining this with the following lemma, $\tilde{g}_{Lk}$ can be approximated by a Transformer network with width at most $12N$, and depth at most $L$. The proof is in \ref{sec9.3}.

\begin{figure}[htbp]
    \centering
    \includegraphics[width=0.6\linewidth]{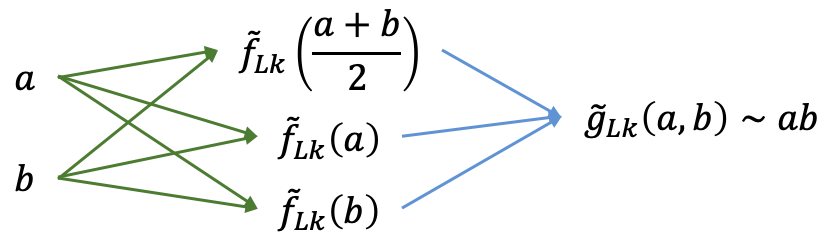}
    \caption{The composition of $g_{Lk}$, where the green and blue lines show $f_{Lk}$ and affine transformations respectively.}
    \label{fig6}
\end{figure}

\begin{lem}
\label{lm5}
    A ReLU FNN with width $N$ and depth $L$, where all the values of inputs, outputs and hidden layers are all non-negative, can be constructed by a Transformer network with width $2N$ and depth $L$.
\end{lem}

\subsection{Approximation of Polynomials}
\label{sec4.2}

We construct a neural network to approximate polynomials.
Since $f_{Lk}(x)$ constructed in Section \ref{sec4.1} is convex and $0\leq f_{Lk}(x) \leq 1$ holds for any $x \in [0,1]$, 
\begin{equation*}
\begin{split}
    \tilde{g}(x,y) = & 2\tilde{f}_{Lk}\left(\frac{x+y}{2}\right) - \frac{1}{2} \tilde{f}_{Lk}(x) - \frac{1}{2}\tilde{f}_{Lk}(y)\\
    = & \tilde{f}_{Lk} \left(\frac{x+y}{2}\right) - \left( \frac{1}{2} \left( \tilde{f}_{Lk}(x) + \tilde{f}_{Lk}(y) \right) - \tilde{f}_{Lk}\left(\frac{x+y}{2}\right) \right) \\
    \leq & \tilde{f}_{Lk} \left(\frac{x+y}{2}\right) \leq 1
\end{split}
\end{equation*}
holds for any $0\leq x,y\leq 1$. Hence, $\tilde{h}_{Lk} (x,y) \coloneqq \mathrm{ReLU}(\tilde{g}_{Lk} (x,y))$ satisfies $0 \leq \tilde{h}_{Lk} (x,y) \leq 1$, and by applying $\tilde{h}_{Lk} (x,y)$ repeatedly, we can obtain an approximation of $\psi (\boldsymbol{x}_1) = x_{11}^{p_{11}} \dots x_{d1}^{p_{d1}}$. As degree $i$ mononomials of $x_{11}, \dots, x_{1d}$ can be written in the format $x_{11}^{p_{11}} \dots x_{d1}^{p_{d1}}$ where $p_{11} + \dots + p_{d1} = i$, the total number of such polynomials is at most
\begin{equation*}
    \binom{i+d-1}{d-1} = \binom{i+d-1}{i} = \frac{i+d-1}{i} \cdot \frac{i+d-2}{i-1} \dots \frac{d}{1} \leq d^i.
\end{equation*}
This implies that the number of monomials of degree $s$ or less of $x_{i1}, \dots, x_{id}$ are at most $d + d^2 + \dots + d^s \leq sd^s$. Note that constructing a monomial of degree $s$ or less requires at most $s-1$ multiplications and $\tilde{h}_{Lk} (x,y)$ has width $12N$ and depth $L+1$, due to the extra layer applying the ReLU function. By approximating all the single-term polynomials in $\boldsymbol{x}_1$, $\mathrm{TF}_0$, the Transformer network approximating all monomials over $\boldsymbol{x}_1$ with degree $s$ or less, can be constructed with width at most $sd^s \cdot 12N = 12sd^sN$, and depth at most $(s-1)(L+1)$.

\subsection{Approximation Error}
\label{sec4.3}
We study an approximation error of the neural networks constrcuted above.

First, we evaluate the approximation error of $\tilde{g}_{Lk}(x,y) \sim xy$. Since 
\begin{equation*}
    \tilde{g}_{Lk}(x,y) - xy = \left(f_{Lk} \left(\frac{x+y}{2} \right) - \left(\frac{x+y}{2}\right)^2 \right) - \frac{1}{2} (\tilde{f}_{Lk}(x) - x^2) - \frac{1}{2} (\tilde{f}_{Lk}(y) - y^2),
\end{equation*}
and $0 \leq \tilde{f}_{Lk} - x^2 \leq \frac{1}{4^{Lk+1}}$ holds for all $x\in [0,1]$, the value of $\tilde{g}_{Lk}(x,y) - xy$ must be at least $\displaystyle 0 - \frac{1}{2} \cdot \frac{1}{4^{Lk+1}} - \frac{1}{2} \cdot \frac{1}{4^{Lk+1}} = -\frac{1}{4^{Lk}+1}$ and be at most $2\displaystyle \cdot \frac{1}{4^{Lk+1}} - 0 - 0 = \frac{1}{2^{2Lk+1}}$. Since we defined $k$ so that $2^k \leq N < 2^{k+1}$ holds, we obtain 
\begin{equation*}
    |\tilde{g}_{Lk}(x,y) - xy| \leq \frac{1}{2^{2Lk+1}} = \frac{1}{2} \left( \frac{1}{2^{2k}} \right)^L \leq \frac{1}{2} \left( \frac{1}{2^{k+1}} \right)^L < \left( \frac{1}{N} \right)^L, 
\end{equation*}
and $xy \geq 0$ for $x,y\in [0,1]$ yields
\begin{equation*}
    |\tilde{h}_{Lk}(x,y) - xy| = |\mathrm{ReLU}(\tilde{g}_{Lk}(x,y)) - xy| \leq |\tilde{g}_{Lk}(x,y) - xy| < N^{-L}.
\end{equation*}

Next, we discuss the approximation error of column-wise monomials. Specifically, we show that the approximation error of a degree-$j\ (\leq s)$ monomial of $\boldsymbol{x}_1$ is at most $(j-1)N^{-L}$, by induction. 

The case for $j=1$ is obvious, because we do not need to take products. Assume the hypothesis is true for degree-$j\ (< s)$. In this case, $\mathcal{E}_{\mathrm{TF}_0} (x_{11}^{p_{11}} \dots p_{d1}^{p_{d1}})$, the approximation error of $x_{11}^{p_{11}} \dots x_{d1}^{p_{d1}}$ by $\mathrm{TF}_0$ satisfies $|\mathcal{E}_{\mathrm{TF}_0} (x_{11}^{p_{11}} \dots x_{d1}^{p_{d1}})| \leq (j-1)N^{-L}$. And because $x_{11}, \dots, x_{d1}\in [0,1]$, the approximation error of $x_{11}^{p_{11}} \dots x_{d1}^{p_{d1}} \cdot x_{i1}$ is at most
\begin{equation*}
\begin{split}
    & |\tilde{g}_{Lk} (x_{11}^{p_{11}} \dots x_{d1}^{p_{d1}} + \mathcal{E}_{\mathrm{TF}_0} (x_{11}^{p_{11}} \dots p_{d1}^{p_{d1}}), x_{i1}) - x_{11}^{p_{11}} \dots x_{d1}^{p_{d1}} x_{i1}| \\
    &\leq  |\tilde{g}_{Lk} (x_{11}^{p_{11}} \dots x_{d1}^{p_{d1}} + \mathcal{E}_{\mathrm{TF}_0} (x_{11}^{p_{11}} \dots p_{d1}^{p_{d1}}), x_{i1})  \\
    & \qquad - (x_{11}^{p_{11}} \dots x_{d1}^{p_{d1}} + \mathcal{E}_{\mathrm{TF}_0} (x_{11}^{p_{11}} \dots p_{d1}^{p_{d1}})) \cdot x_{i1}|\\ 
    & \quad +  |(x_{11}^{p_{11}} \dots x_{d1}^{p_{d1}} + \mathcal{E}_{\mathrm{TF}_0} (x_{11}^{p_{11}} \dots p_{d1}^{p_{d1}})) \cdot x_{i1} - x_{11}^{p_{11}} \dots x_{d1}^{p_{d1}} \cdot x_{i1}|\\
    &\leq  N^{-L} + |\mathcal{E}_{\mathrm{TF}_0} (x_{11}^{p_{11}} \dots p_{d1}^{p_{d1}})| \leq N^{-L} + (j-1)N^{-L} = j \cdot N^{-L},
\end{split}
\end{equation*}
which completes the induction. Hence, the approximation error of $m_{\boldsymbol{p}_1} (\boldsymbol{X}) = \boldsymbol{x}_1^{\boldsymbol{p}_1} + \dots + \boldsymbol{x}_n^{\boldsymbol{p}_1}$ by $\mathrm{TF}_1$ is at most $n(|\boldsymbol{p}_1| - 1) N^{-L}$, proving Proposition \ref{prop2}.

\section{Proof of Proposition \ref{prop2}}
\label{sec4.4}
Since feed-forward layers affect each columns independently, 
\begin{equation}
\label{eq1}
    \mathrm{TF}_0
    \begin{pmatrix}
        x_{11} & x_{12} & \cdots & x_{1n} \\
        \vdots & \vdots & \ddots & \vdots \\
        x_{d1} & x_{d2} & \cdots & x_{dn} \\
        0 & 0 & \cdots & 0 \\
        0 & 0 & \cdots & 0 \\
        \vdots & \vdots & \ddots & \vdots \\
        0 & 0 & \cdots & 0 \\
    \end{pmatrix}
    \sim  
    \begin{bmatrix}
        x_{11} & x_{12} & \cdots & x_{1n} \\
        \vdots & \vdots & \ddots & \vdots \\
        x_{d1} & x_{d2} & \cdots & x_{dn} \\
        x_{11}^2 & x_{12}^2 & \cdots & x_{1n}^2 \\
        x_{11}x_{21} & x_{12}x_{22} & \cdots & x_{1n}x_{2n} \\
        \vdots & \vdots & \ddots & \vdots\\
        x_{d1}^s & x_{d2}^s & \cdots & x_{dn}^s \\
    \end{bmatrix}
\end{equation}
holds: i.e. the monomials of $\boldsymbol{x}_1, \boldsymbol{x}_2, \dots, \boldsymbol{x}_n$ are simultaneously produced in the $1,2,\dots,n$th columns respectively. Let $d'$ be the number of rows in $\mathrm{TF}_0$. By setting $m = d',\ h = 1,\ \boldsymbol{W}_O^1 = I_{d'},\ \boldsymbol{W}_V^1 = (n+1) \boldsymbol{I}_{d'},\ \boldsymbol{W}_K^1 = \boldsymbol{W}_Q^1 = \boldsymbol{O}^{d'}$ in the attention layer, we obtain 
\begin{equation*}
    \begin{split}
        \mathrm{Attn} \left( \boldsymbol{Y} = 
        \begin{bmatrix}
        y_{11} & y_{12} & \cdots & y_{1n} & 0\\
        y_{21} & y_{22} &\cdots & y_{2n} & 0\\
        \vdots & \vdots & \ddots & \vdots & \vdots \\
        y_{d'1} & y_{d2} & \cdots & y_{d'n} & 0\\
        \end{bmatrix}
        \right)
        &=  \boldsymbol{Y} + (n+1) \boldsymbol{Y} \cdot \mathrm{softmax} (\boldsymbol{O}_{n+1})\\
        &=  \boldsymbol{Y} + (n+1) \boldsymbol{Y} \cdot \frac{1}{n+1} \boldsymbol{1}_{n+1} \\
        &= 
        \begin{bmatrix}
        y_{11} + s_1 & y_{12} + s_1 & \cdots & y_{1n} + s_1 & s_1\\
        y_{21} + s_2 & y_{22} + s_2 &\cdots & y_{2n} + s_2 & s_2\\
        \vdots & \vdots & \ddots & \vdots & \vdots \\
        y_{d'1} + s_{d'} & y_{d'2} + s_{d'} & \cdots & y_{d'n} + s_{d'} & s_{d'}\\
        \end{bmatrix}, 
    \end{split}
\end{equation*}
where $s_i = y_{i1} + y_{i2} + \dots + y_{in}$. Hence, if we define $\mathrm{TF}_1$ as the composition of $\mathrm{TF}_0$ and the above attention layer, equation \eqref{eq1} yields
\begin{equation*}
    \mathrm{TF_1}
    \begin{bmatrix}
        x_{11} & x_{12} & \cdots & x_{1n} & 0\\
        \vdots & \vdots & \ddots & \vdots & 0\\
        x_{d1} & x_{d2} &\cdots & x_{dn} & 0\\
        0 & 0 & \cdots & 0 & 0\\
        0 & 0 & \cdots & 0 & 0\\
        \vdots & \vdots & \ddots & \vdots & \vdots\\
        0 & 0 & \cdots & 0 & 0\\
    \end{bmatrix}_{n+1}
    \sim 
    \begin{bmatrix}
        m_{(1,0,\dots,0)}(\boldsymbol{X})\\
        \vdots \\
        m_{(0,\dots,0,1)}(\boldsymbol{X})\\
        m_{(2,0,\dots,0)}(\boldsymbol{X})\\
        m_{(1,1,0,\dots,0)}(\boldsymbol{X})\\
        \vdots\\
        m_{(0,\dots,0,s)}(\boldsymbol{X})\\
    \end{bmatrix} ,
\end{equation*}
since $m_{\boldsymbol{p}_1}(\boldsymbol{X}) = \boldsymbol{x}_1^{\boldsymbol{p}_1} + \dots + \boldsymbol{x}_n^{\boldsymbol{p}_n}$ ($[\boldsymbol{A}]_{i}$ denotes the $i$-th column of matrix $\boldsymbol{A}$). Hence, we have constructed all the rank-1 monomial column-symmetric polynomials of $\boldsymbol{X}$, with degree $s$ or less. As $\mathrm{TF}_1$ is a composition of $\mathrm{TF}_0$ and a single-head attention layer, the width and depth of $\mathrm{TF}_1$ are at most $12sd^sN$ and $(s-1)(L+1) + 1$.

Since the approximation error of each degree-$j$ monomial was at most $(j-1)N^{-L}$, according to Section \ref{sec4.3}, the approximation error of $m_{\boldsymbol{p}_1} (\boldsymbol{X}) = \boldsymbol{x}_1^{\boldsymbol{p}_1} + \dots + \boldsymbol{x}_n^{\boldsymbol{p}_1}$ by $\mathrm{TF}_1$ is at most $n(|\boldsymbol{p}_1| - 1) N^{-L}$, as this polynomial is consisted of $n$ degree-$|\boldsymbol{p}_1|$ terms. This proves Proposition \ref{prop2}.

\section{Proof of Proposition \ref{prop3}}
\label{sec5}
From the following lemma, we can obtain rank-$(r+1)$ monomial column-symmetric polynomials from rank-$r$ and rank-$1$ monomial column-symmetric polynomials, by addition and multiplication. 
\begin{lem}
\label{lm6}
    For $\boldsymbol{p}_1, \dots, \boldsymbol{p}_{r+1} \in \mathbb{N}^d$, the following equality holds:
    \begin{equation}
    \label{eq2}
    \begin{split}
        m_{\boldsymbol{p}_1, \dots, \boldsymbol{p}_r, \boldsymbol{p}_{r+1}}(\boldsymbol{X})
        &=  m_{\boldsymbol{p}_1, \dots, \boldsymbol{p}_r} (\boldsymbol{X}) \cdot m_{\boldsymbol{p}_{r+1}}(\boldsymbol{X}) - m_{\boldsymbol{p}_1 + \boldsymbol{p}_{r+1}, \boldsymbol{p}_2, \dots, \boldsymbol{p}_r} (\boldsymbol{X}) \\
        & \quad -  m_{\boldsymbol{p}_1, \boldsymbol{p}_2 + \boldsymbol{p}_{r+1}, \boldsymbol{p}_3, \dots, \boldsymbol{p}_r} (\boldsymbol{X})
        - \cdots - m_{\boldsymbol{p}_1, \dots \boldsymbol{p}_{r-1}, \boldsymbol{p}_r + \boldsymbol{p}_{r+1}} (\boldsymbol{X}).
    \end{split}
    \end{equation}
\end{lem}

\begin{proof}
    The lemma can be proved by multiplying a rank-$r$ monomial column-symmetric polynomial with a rank-$1$ monomial column-symmetric polynomial, and subtracting the extra terms. If $r+1 \leq n$,
    \begin{equation*}
    \begin{split}
        & m_{\boldsymbol{p}_1, \dots, \boldsymbol{p}_{r+1}}(\boldsymbol{X}) = \sum_{\sigma \in S_n} \frac{1}{(n-r-1)!} \boldsymbol{x}_{\sigma(1)}^{\boldsymbol{p}_1} \dots \boldsymbol{x}_{\sigma(r+1)}^{\boldsymbol{p}_{r+1}}\\
        = & \sum_{\sigma \in S_n} \frac{1}{(n-r)!} \boldsymbol{x}_{\sigma(1)}^{\boldsymbol{p}_1} \dots \boldsymbol{x}_{\sigma(r)}^{\boldsymbol{p}_{r}} \cdot (\boldsymbol{x}_1^{\boldsymbol{p}_{r+1}} + \dots + \boldsymbol{x}_n^{\boldsymbol{p}_{r+1}} - \boldsymbol{x}_{\sigma(1)}^{\boldsymbol{p}_{r+1}} - \dots - \boldsymbol{x}_{\sigma(r)}^{\boldsymbol{p}_{r+1}})\\
        = & \sum_{\sigma \in S_n} \frac{1}{(n-r)!} \left( \boldsymbol{x}_{\sigma(1)}^{\boldsymbol{p}_1} \dots \boldsymbol{x}_{\sigma(r)}^{\boldsymbol{p}_{r}} \cdot m_{\boldsymbol{p}_{r+1}}(\boldsymbol{X}) - \boldsymbol{x}_{\sigma(1)}^{\boldsymbol{p}_1 + \boldsymbol{p}_{r+1}} \dots \boldsymbol{x}_{\sigma(r)}^{\boldsymbol{p}_{r}} - \dots - \boldsymbol{x}_{\sigma(1)}^{\boldsymbol{p}_1} \dots \boldsymbol{x}_{\sigma(r)}^{\boldsymbol{p}_{r} + \boldsymbol{p}_{r+1}} \right)\\
        = & m_{\boldsymbol{p}_1, \dots, \boldsymbol{p}_k}(\boldsymbol{X}) \cdot m_{\boldsymbol{p}_{r+1}}(\boldsymbol{X})
        - m_{\boldsymbol{p}_1 + \boldsymbol{p}_{r+1}, \boldsymbol{p}_2, \dots, \boldsymbol{p}_r} (\boldsymbol{X})
        - \cdots - m_{\boldsymbol{p}_1, \dots \boldsymbol{p}_{r-1}, \boldsymbol{p}_r + \boldsymbol{p}_{r+1}} (\boldsymbol{X})
    \end{split}
    \end{equation*}
    holds (note that the coefficient becomes $\frac{1}{(n-r-1)!} \cdot \frac{1}{n-r} = \frac{1}{(n-r)!}$ in the second row, as the term
    \begin{equation*}
    \begin{split}
        & \boldsymbol{x}_{\sigma(1)}^{\boldsymbol{p}_1} \dots \boldsymbol{x}_{\sigma(r)}^{\boldsymbol{p}_{r}} \cdot (\boldsymbol{x}_1^{\boldsymbol{p}_{r+1}} + \dots + \boldsymbol{x}_n^{\boldsymbol{p}_{r+1}} - \boldsymbol{x}_{\sigma(1)}^{\boldsymbol{p}_{r+1}} - \dots - \boldsymbol{x}_{\sigma(r)}^{\boldsymbol{p}_{r+1}}) \\
        & =  \sum_{1\leq j\leq n, j\neq \sigma(1), \dots, \sigma(r)} \boldsymbol{x}_{\sigma(1)}^{\boldsymbol{p}_1} \dots \boldsymbol{x}_{\sigma(r)}^{\boldsymbol{p}_{r}} \cdot \boldsymbol{x}_{\sigma(j)}^{\boldsymbol{p}_{r+1}}
    \end{split}
    \end{equation*}
    corresponds to the sum over $n-r$ elements in $S_n$). Next, if $r = n$, then
    \begin{equation*}
    \begin{split}
        & m_{\boldsymbol{p}_1, \dots, \boldsymbol{p}_r}(\boldsymbol{X}) \cdot m_{\boldsymbol{p}_{r+1}}(\boldsymbol{X})\\
        &=  \left( \sum_{\sigma \in S_n} \frac{1}{(n-r)!} \boldsymbol{x}_{\sigma(1)}^{\boldsymbol{p}_1} \dots \boldsymbol{x}_{\sigma(r)}^{\boldsymbol{p}_r} \right) \cdot (\boldsymbol{x}_1^{\boldsymbol{p}_{r+1}} + \dots + \boldsymbol{x}_r^{\boldsymbol{p}_{r+1}})\\
        &=  \sum_{\sigma} \frac{1}{(n-r)!} \boldsymbol{x}_{\sigma(1)}^{\boldsymbol{p}_1} \dots \boldsymbol{x}_{\sigma(r)}^{\boldsymbol{p}_r} \boldsymbol{x}_{\sigma(1)}^{\boldsymbol{p}_1} + \dots + \sum_{\sigma} \frac{1}{(n-r)!} \boldsymbol{x}_{\sigma(1)}^{\boldsymbol{p}_{r+1}} \dots \boldsymbol{x}_{\sigma(r)}^{\boldsymbol{p}_r} \cdot \boldsymbol{x}_{\sigma(r)}^{\boldsymbol{p}_{r+1}}\\
        &=  m_{\boldsymbol{p}_1 + \boldsymbol{p}_{r+1}, \boldsymbol{p}_2, \dots, \boldsymbol{p}_r} (\boldsymbol{X}) + \cdots + m_{\boldsymbol{p}_1, \dots \boldsymbol{p}_{r-1}, \boldsymbol{p}_r + \boldsymbol{p}_{r+1}} (\boldsymbol{X}), 
    \end{split}
    \end{equation*}
    implying that both sides of equation \eqref{eq2} are equal to 0. Last, if $r > n$, both sides of equation \ref{eq2} are equal to 0, since monomial column-symmetric polynomials with rank-$r$ or greater are all equal to 0 from its definition.
\end{proof}

Assume that a Transformer $\mathcal{T}_r$ approximates all monomial column-symmetric polynomials with rank-$r$ or less. Now, we denote $\mathcal{T}_r (m_{\boldsymbol{p}_1, \dots, \boldsymbol{p}_r})(\boldsymbol{X})$ as the approximation of $m_{\boldsymbol{p}_1, \dots, \boldsymbol{p}_r} (\boldsymbol{X})$ by $\mathcal{T}_r$. If $0 \leq \mathcal{T}_r (m_{\boldsymbol{p}_1, \dots, \boldsymbol{p}_r}) (\boldsymbol{X}) \leq P(n,r)$ and $0 \leq \mathcal{T}_r (m_{\boldsymbol{p}_{r+1}}) (\boldsymbol{X}) \leq n$, we can approximate rank-$(r+1)$ monomial column-symmetric polynomials $m_{\boldsymbol{p}_1, \dots, \boldsymbol{p}_r, \boldsymbol{p}_{r+1}} (\boldsymbol{X})$ by
\begin{equation}
\label{eq3}
\begin{split}
    m_{\boldsymbol{p}_1, \dots, \boldsymbol{p}_r, \boldsymbol{p}_{r+1}} (\boldsymbol{X}) \
    &\sim  \ nP(n,r) \cdot \tilde{g}_{Lk} \left( \frac{\mathcal{T}_r (m_{\boldsymbol{p}_1, \dots, \boldsymbol{p}_r}) (\boldsymbol{X})}{P(n,r)}, \frac{\mathcal{T}_r (m_{\boldsymbol{p}_{r+1}}) (\boldsymbol{X})}{n} \right) \\
    &  - \mathcal{T}_r (m_{\boldsymbol{p}_1 + \boldsymbol{p}_{r+1}, \boldsymbol{p}_2, \dots, \boldsymbol{p}_r}) (\boldsymbol{X}) - \cdots - \mathcal{T}_r (m_{\boldsymbol{p}_1, \dots \boldsymbol{p}_{r-1}, \boldsymbol{p}_r + \boldsymbol{p}_{r+1}}) (\boldsymbol{X}).
\end{split}
\end{equation}

Thus, all monomial column-symmetric polynomials of rank-$s$ or less can be constructed inductively. In the remainder of this section, we evaluate the Transformer network that approximates monomial column-symmetric polynomials of higher ranks, given that the approximation of rank-1 column-symmetric polynomials is provided as input. We evaluate its size in Section \ref{sec5.1} and its approximation error in Section \ref{sec5.2}.

\subsection{Size of Transformer Network}
\label{sec5.1}
By the result above, when the approximations of monomial column-symmetric polynomials of rank-$r$ or less are given as inputs, we can construct the right-hand side of equation \eqref{eq3} using a Transformer with width $12N$ and depth $L$, since applying $\tilde{g}_{Lk}$ requires such a size (note that addition, subtraction, and scaling of terms can be achieved by adjusting the parameter matrix $\boldsymbol{W}_2$ in the final feed-forward layer). In addition, when constructing approximations of rank-$r$ monomial column-symmetric polynomials, we apply the functions $\mathrm{ReLU} (x)$ and $x-\mathrm{ReLU}(x - P(n,r))$ at the end to ensure that the values of the Transformer approximation of rank-$r$ monomial column-symmetric polynomials lie within the range $[0, P(n,r)]$. (note that the maximum value of rank-$r$ monomial column-symmetric polynomials is
\begin{equation*}
    \sum_{\sigma \in S_n} \frac{1}{(n-r)!} = \frac{n!}{(n-r)!} = P(n,r),
\end{equation*}
as $S_n$ has $n!$ elements). As a result, approximating monomial column-symmetric polynomials of rank-$r+1$ from those of ranks $r$ and $1$ can be achieved by a Transformer network with width $12N$ and depth $L+2$ per polynomial.

Next, we consider the total number of monomial column-symmetric polynomials of rank-$r\ (\leq s)$. Consider a rank-$r$ monomial column-symmetric polynomial containing the term $\boldsymbol{x}_1^{\boldsymbol{p}_1} \dots \boldsymbol{x}_r^{\boldsymbol{p}_r}\ (|\boldsymbol{p}_1| + \dots + |\boldsymbol{p}_r| \leq s, \ \boldsymbol{p}_1 > \dots > \boldsymbol{p}_r > \boldsymbol{0}_d)$. From Lemma \ref{lm2}, the possible combinations of $(|\boldsymbol{p}_1|, \dots, |\boldsymbol{p}_r|)$ is $\displaystyle \binom{s}{r}$, as $|\boldsymbol{p}_1|, \dots, |\boldsymbol{p}_r| \geq 1$. 

In addition, consider the total number of combinations of $\boldsymbol{p}_1, \dots, \boldsymbol{p}_r$, where $|\boldsymbol{p}_1|, \dots, |\boldsymbol{p}_r|$ are fixed. For each $j = 1, \dots, r$, the number of solutions $(p_{1j}, \dots, p_{dj})$ satisfying the equation $p_{1j} + \dots + p_{dj} = |\boldsymbol{p}_j|$ with $p_{1j}, \dots, p_{dj} \geq 0$, by Lemma \ref{lm2}, is at most 
\begin{equation*}
    \binom{|\boldsymbol{p}_j| + d - 1}{d - 1} = \binom{|\boldsymbol{p}_j| + d - 1}{|\boldsymbol{p}_j|} = \frac{|\boldsymbol{p}_j| + d - 1}{|\boldsymbol{p}_j|} \frac{|\boldsymbol{p}_j| + d - 2}{|\boldsymbol{p}_j| - 1} \dots \frac{d}{1} \leq d^{|\boldsymbol{p}_j|}.
\end{equation*}
Since the product over all $j$-s yields $d^{|\boldsymbol{p}_1|} \cdots d^{|\boldsymbol{p}_r|} = d^{|\boldsymbol{p}_1| + \dots + |\boldsymbol{p}_r|} \leq d^s$, the total number of degree-$r$  monomial symmetric polynomials is at most $d^s \cdot \binom{s}{r}$. Hence, to construct a single degree-$(r+1)$ monomial column-symmetric polynomial from degree-$r$ monomial column-symmetric polynomials, the required width is at most
\begin{equation*}
    12N \cdot d^s \binom{s}{r} = 12N \cdot d^s 2^s \leq 12 \cdot (2d)^s N,
\end{equation*}
(note we have used Lemma \ref{lm1} in the first inequality) and the depth is at most $L+2$. Hence, $\mathrm{TF}_2$ can be realized with a Transformer network having width at most $12\cdot (2d)^s N$ and depth at most $(s-1)(L+2)$.

\subsection{Approximation Error}
\label{sec5.2}
In this subsection, we will provide an upper bound of the approximation error of monomial column-symmetric polynomials by $\mathcal{T} \coloneqq \mathrm{TF}_1 \circ \mathrm{TF}_2$. We will prove by induction, based on Lemma \ref{lm6}. 

According to Section \ref{sec4.3}, the approximation error of the rank-1 monomial column-symmetric polynomial $m_{\boldsymbol{p}_1}$ was at most $n(|\boldsymbol{p}_1| - 1)N^{-L} < P(n, 1) \cdot (|\boldsymbol{p}_1| + 1)N^{-L}$. Hence the assumption holds for $r = 1$. Next, assume that the approximation error of $m_{\boldsymbol{p}_1, \dots, \boldsymbol{p}_r} (\boldsymbol{X})$ in $[0,1]^{d \times n}$ is at most $(P(n+r-1, r) \cdot (|\boldsymbol{p}_1| + 1) \cdots (|\boldsymbol{p}_r| + 1) - n^r) N^{-L}$, for any $\boldsymbol{p}_1, \dots, \boldsymbol{p}_r > \boldsymbol{0}_d$ such that $|\boldsymbol{p}_1| + \dots + |\boldsymbol{p}_r| \leq s$. From equation \eqref{eq3}, the approximation error of a rank-$(r+1)$ monomial column-symmetric polynomial $m_{\boldsymbol{p}_1, \dots, \boldsymbol{p}_r, \boldsymbol{p}_{r+1}}(\boldsymbol{X})$ is at most
\begin{equation*}
\begin{split}
    & nP(n,r) \Bigg(N^{-L} + \Bigg| \frac{m_{\boldsymbol{p}_1, \dots, \boldsymbol{p}_r}(\boldsymbol{X}) + \mathcal{E}_{\mathcal{T}} (m_{\boldsymbol{p}_1, \dots, \boldsymbol{p}_r}(\boldsymbol{X}))}{P(n,r)} \cdot \frac{m_{\boldsymbol{p}_{r+1}} (\boldsymbol{X}) + \mathcal{E}_{\mathcal{T}} (m_{\boldsymbol{p}_{r+1}}(\boldsymbol{X}))}{n}\\
    & \quad - \frac{m_{\boldsymbol{p}_1, \dots, \boldsymbol{p}_r} (\boldsymbol{X}) m_{\boldsymbol{p}_{r+1}} (\boldsymbol{X})}{nP(n,r)} \Bigg| \Bigg) \\
    & + |\mathcal{E}_{\mathcal{T}} (m_{\boldsymbol{p}_1 + \boldsymbol{p}_{r+1}, \boldsymbol{p}_2, \dots, \boldsymbol{p}_r}(\boldsymbol{X}))| + \dots + |\mathcal{E}_{\mathcal{T}} (m_{\boldsymbol{p}_1, \dots, \boldsymbol{p}_{r-1}, \boldsymbol{p}_r + \boldsymbol{p}_{r+1}}(\boldsymbol{X}))|,
\end{split}
\end{equation*}
as the approximation error of $\phi (x,y) = xy\ (x,y\in [0,1])$ is at most $N^{-L}$ in $[0,1]^{d\times n}$. 
Since the maximum values of $m_{\boldsymbol{p}_1, \dots, \boldsymbol{p}_r}$ and in $[0,1]^{d\times n}$ are $P(n,r)$ and $n$, we obtain the the upper bound of the approximation error by performing some algebra, which is
\begin{equation}
\label{eq4}
\begin{split}
    &n|\mathcal{E}_{\mathcal{T}} (m_{\boldsymbol{p}_1, \dots, \boldsymbol{p}_r}(\boldsymbol{X}))| + P(n,r) |\mathcal{E}_{\mathcal{T}} (m_{\boldsymbol{p}_{r+1}}(\boldsymbol{X}))| \\
    &+ |\mathcal{E}_{\mathcal{T}} (m_{\boldsymbol{p}_1, \dots, \boldsymbol{p}_r}(\boldsymbol{X})) \mathcal{E}_{\mathcal{T}} (m_{\boldsymbol{p}_{r+1}}(\boldsymbol{X}))| + nP(n,r) N^{-L} \\
    &+ |\mathcal{E}_{\mathcal{T}} (m_{\boldsymbol{p}_1 + \boldsymbol{p}_{r+1}, \boldsymbol{p}_2, \dots, \boldsymbol{p}_r}(\boldsymbol{X}))| + \cdots + |\mathcal{E}_{\mathcal{T}} (m_{\boldsymbol{p}_1, \dots, \boldsymbol{p}_{r-1}, \boldsymbol{p}_r + \boldsymbol{p}_{r+1}}(\boldsymbol{X}))|.
\end{split}
\end{equation}
Now, the approximation errors of $m_{\boldsymbol{p}_1, \dots, \boldsymbol{p}_r}$ and $m_{\boldsymbol{p}_{r+1}}$ in $[0,1]^{d\times n}$ by $\mathcal{T}$ are at most $(P(n+r-1, r) \cdot (|\boldsymbol{p}_1| + 1) \cdots (|\boldsymbol{p}_r| + 1) - n^r) N^{-L}$ and $n(|\boldsymbol{p}_{r+1}| - 1) N^{-L}$ (the former follows from the assumption of the induction, and the latter from Section \ref{sec4.3}). Hence, the first 2 terms of equation \eqref{eq4} are at most
\begin{equation*}
\begin{split}
    & n\cdot (P(n+r-1, r) \cdot (|\boldsymbol{p}_1| + 1) \dots (|\boldsymbol{p}_r| + 1) - n^r) N^{-L} + P(n,r) \cdot (n |\boldsymbol{p}_{r+1}| - n) N^{-L}\\
    & \leq  n\cdot (P(n+r-1, r) \cdot (|\boldsymbol{p}_1| + 1) \dots (|\boldsymbol{p}_r| + 1) - n^r) N^{-L} + n^{r+1} \cdot (|\boldsymbol{p}_{r+1}| - 1) N^{-L}.
\end{split}
\end{equation*}
The next term $|\mathcal{E}_{\mathcal{T}} (m_{\boldsymbol{p}_1, \dots, \boldsymbol{p}_r}(\boldsymbol{X})) \mathcal{E}_{\mathcal{T}} (m_{\boldsymbol{p}_{r+1}}(\boldsymbol{X}))|$ in equation \eqref{eq4} is at most
\begin{equation*}
\begin{split}
    & (P(n+r-1, r) \cdot (|\boldsymbol{p}_1| + 1) \dots (|\boldsymbol{p}_r| + 1) - n^r) N^{-L} \cdot n(|\boldsymbol{p}_{r+1}| - 1) N^{-L} \\
    & \leq  n \cdot (P(n+r-1, r) \cdot (|\boldsymbol{p}_1| + 1) \dots (|\boldsymbol{p}_r| + 1) (|\boldsymbol{p}_{r+1}| - 1) - n^r (|\boldsymbol{p}_{r+1}| - 1)) N^{-L}.
\end{split}
\end{equation*}
From these inequalities, the first 3 terms in equation \eqref{eq4} is bounded by
\begin{equation*}
\begin{split}
    & n\cdot (P(n+r-1, r) \cdot (|\boldsymbol{p}_1| + 1) \dots (|\boldsymbol{p}_r| + 1) - n^r) N^{-L} + n^{r+1} \cdot (|\boldsymbol{p}_{r+1}| - 1) N^{-L} \\
    & + n \cdot (P(n+r-1, r) \cdot (|\boldsymbol{p}_1| + 1) \dots (|\boldsymbol{p}_r| + 1) (|\boldsymbol{p}_{r+1}| - 1) - n^r (|\boldsymbol{p}_{r+1}| - 1)) N^{-L} \\
    &= (n\cdot P(n+r-1,r) \cdot (|\boldsymbol{p}_1| + 1) \dots (|\boldsymbol{p}_r| + 1) |\boldsymbol{p}_{r+1}| - n^{r+1}) N^{-L}.
\end{split}
\end{equation*}
In addition, by using the inequality $a+b+1 < (a+1)(b+1)$ for $a,b\in \mathbb{N}_+$, the upper bound of the remaining terms in \eqref{eq4} (i.e. $|\mathcal{E}_{\mathcal{T}} (m_{\boldsymbol{p}_1 + \boldsymbol{p}_{r+1}, \boldsymbol{p}_2, \dots, \boldsymbol{p}_r}(\boldsymbol{X}))| + \cdots + |\mathcal{E}_{\mathcal{T}} (m_{\boldsymbol{p}_1, \dots, \boldsymbol{p}_{r-1}, \boldsymbol{p}_r + \boldsymbol{p}_{r+1}}(\boldsymbol{X}))|$) is
\begin{equation*}
\begin{split}
    & (P(n+r-1, r) \cdot (|\boldsymbol{p}_1| + |\boldsymbol{p}_{r+1}| + 1) (|\boldsymbol{p}_2| + 1) \dots (|\boldsymbol{p}_r| + 1) - n^r) N^{-L}\\
    &\quad  + \dots + (P(n+r-1, r) \cdot (|\boldsymbol{p}_1| + 1) \dots (|\boldsymbol{p}_{r-1}| + 1) (|\boldsymbol{p}_r| + |\boldsymbol{p}_{r+1}| + 1) - n^r) N^{-L}\\
    &<  r \cdot ((|\boldsymbol{p}_1| + 1) \dots (|\boldsymbol{p}_r| + 1) (|\boldsymbol{p}_{r+1}| + 1) - n^r) N^{-L}.
\end{split}
\end{equation*}
Hence, equation \eqref{eq4}, the approximation error of $m_{\boldsymbol{p}_1, \dots, \boldsymbol{p}_{r+1}}(\boldsymbol{X})$ in $[0,1]^{d\times n}$, is at most 
\begin{equation*}
\begin{split}
    & (nP(n+r-1, r) \cdot (|\boldsymbol{p}_1| + 1) \dots (|\boldsymbol{p}_r| + 1) |\boldsymbol{p}_{r+1}| - n^{r+1}) N^{-L} + nP(n,r) \cdot N^{-L}\\
    & \quad  + r \cdot (P(n+r-1, r) \cdot (|\boldsymbol{p}_1| + 1) \dots (|\boldsymbol{p}_r| + 1) (|\boldsymbol{p}_{r+1}| + 1) - n^r) N^{-L}\\
    &<  n\cdot (P(n+r-1, r) \cdot (|\boldsymbol{p}_1| + 1) \dots (|\boldsymbol{p}_r| + 1) (|\boldsymbol{p}_{r+1}| + 1) - n^{r+1}) N^{-L} \\
    & \quad - nP(n+1-r)\cdot N^{-L}+ nP(n,r)\cdot N^{-L} \\
    & \quad + r \cdot P(n+r-1, r) \cdot (|\boldsymbol{p}_1| + 1) \dots (|\boldsymbol{p}_r| + 1) (|\boldsymbol{p}_{r+1}| + 1) N^{-L}\\
    &\leq (P(n+r, r+1) \cdot (|\boldsymbol{p}_1| + 1) \dots (|\boldsymbol{p}_r| + 1) (|\boldsymbol{p}_{r+1}| + 1) - n^{r+1}) N^{-L},
\end{split}
\end{equation*}
which completes the induction. Thus, we have proved Proposition \ref{prop3}.

\section{Combining Pieces to Prove Theorem \ref{th1}}
\label{sec6}
According to Section \ref{sec5.2}, the approximation error for the rank-$r$ column-monomial symmetric polynomial $c_{\boldsymbol{p}_1, \dots, \boldsymbol{p}_r} \cdot m_{\boldsymbol{p}_1, \dots, \boldsymbol{p}_r}(\boldsymbol{X})$ is at most
\begin{equation}
\label{eq5}
    P(n+r-1, r) \cdot (|\boldsymbol{p}_1| + 1) \dots (|\boldsymbol{p}_r| + 1) N^{-L}.
\end{equation}
Note that when $x_{11} = \dots = x_{dn} = 1$, the value of the weighted column-monomial symmetric polynomial becomes $P(n,r)$. Dividing the approximation error \eqref{eq5} by $c_{\boldsymbol{p}_1, \dots, \boldsymbol{p}_r} \cdot P(n,r)$, we obtain
\begin{equation*}
\begin{split}
    &\frac{1}{P(n, r)} P(n+r-1, r) \cdot (|\boldsymbol{p}_1| + 1) \dots (|\boldsymbol{p}_r| + 1) N^{-L} \\
    & = \frac{n+r-1}{n} \cdot \frac{n+r-2}{n-1} \dots \frac{n}{n-r+1} \cdot (|\boldsymbol{p}_1| + 1) \dots (|\boldsymbol{p}_r| + 1) N^{-L} \\
    &= \left( 1 + \frac{r-1}{n} \right) \left( 1 + \frac{r-1}{n-1} \right) \dots \left( 1 + \frac{r-1}{n-r+1} \right) \cdot (|\boldsymbol{p}_1| + 1) \dots (|\boldsymbol{p}_r| + 1) N^{-L}.
\end{split}
\end{equation*}
If $n \geq r$, which is when the monomial column-symmetric polynomial is non-zero, and as long as $r$ is fixed, the terms $1 + \frac{r-1}{n}, 1 + \frac{r-1}{n-1}, \dots, 1 + \frac{r-1}{n-r+1}$ monotonically decrease as $n$ gets larger. Hence, as $|\boldsymbol{p}_1| + |\boldsymbol{p}_2| + \dots + |\boldsymbol{p}_r| \leq s$, the formula above is bounded by
\begin{equation*}
\begin{split}
    &\frac{1}{P(r, r)} P(r+r-1, r) \cdot (|\boldsymbol{p}_1| + 1) \dots (|\boldsymbol{p}_r| + 1) N^{-L} \\
    &= \binom{2r-1}{r} (|\boldsymbol{p}_1| + 1) \dots (|\boldsymbol{p}_r| + 1) \\
    &\leq 2^{2r-1} \cdot 2^s \leq 2^{2s-1} \cdot 2^s < 2^{3s},
\end{split}
\end{equation*}
where the first inequality follows from the inequality $a+b+1 \leq (a+1)(b+1)$ for $a,b \in \mathbb{N}_+$, implying the maximum of $(|\boldsymbol{p}_1| + 1) \dots (|\boldsymbol{p}_r| + 1)$ is achieved when $r = s$ and $|\boldsymbol{p}_1| = \dots =  |\boldsymbol{p}_r| = 1$. Thus, we obtain the approximation error for weighted rank-$r$ monomial column-symmetric polynomials as
\begin{align*}
    |\mathcal{E}_{\mathrm{TF}_1 \circ \mathrm{TF}_2} (m_{\boldsymbol{p}_1, \dots, \boldsymbol{p}_r}(\boldsymbol{X}))| &\leq 8^s N^{-L} \cdot P(n,r) \\
    &= 8^s N^{-L} \cdot m_{\boldsymbol{p}_1, \dots, \boldsymbol{p}_r} (\boldsymbol{1}_{d\times n}).
\end{align*}
From Lemma \ref{lm3}, any permutation symmetric function  $f(\boldsymbol{X})$ can be expressed as a weighted sum of monomial symmetric polynomials: i.e. there exist coefficients $c_{\boldsymbol{p}_1, \dots, \boldsymbol{p}_r}$ such that 
\begin{equation*}
    f(\boldsymbol{X}) = \sum_{1\leq r\leq s, \boldsymbol{p}_1 \geq \dots \geq \boldsymbol{p}_r} c_{\boldsymbol{p}_1, \dots, \boldsymbol{p}_r} m_{\boldsymbol{p}_1, \dots, \boldsymbol{p}_r} (\boldsymbol{X}).
\end{equation*}
Thus, by taking the sum of all weighted monomial column-symmetric polynomials, we obtain the upper bound of the approximation error for $f(\boldsymbol{X})$ as 
\begin{equation*}
    \sum_{1\leq r\leq s, \boldsymbol{p}_1 \geq \dots \geq \boldsymbol{p}_r} c_{\boldsymbol{p}_1, \dots, \boldsymbol{p}_r} \cdot 8^s N^{-L} \cdot m_{\boldsymbol{p}_1, \dots, \boldsymbol{p}_r} (\boldsymbol{1}_{d\times n}) = 8^s N^{-L} \cdot f(\boldsymbol{1}_{d\times n}).
\end{equation*}
Now, $\mathrm{TF}_1 \circ \mathrm{TF}_2$ can be constructed by a Transformer network with width at most $\max (12sd^s, 12(2d)^s) = 12(2d)^s$ and depth at most $(s-1)(L+1) + 1 + (s-1)(L+2) = (s-1)(2L+3) + 1 < 2sL + 3s$. Since taking the weighted sum of inputs on the same column can be achieved by the feed-forward layer, $\mathrm{TF}$ can be constructed by altering the parameters of the last feed-forward layer of $\mathrm{TF}_1 \circ \mathrm{TF}_2$. Hence, we have proved Theorem \ref{th1}.

\section{Discussions}
\label{sec7}

\subsection{Transformer vs. Neural networks}
In this paper, we have utilized the parameter efficiency of Transformers, specifically the efficiency of the attention layer and the parallel processing capability of the feed-forward layer. Since we used only the attention layer to compute the row-wise sum of inputs, the entire process in this paper can be implemented using neural networks by treating the input $\boldsymbol{X} \in [0,1]^{d\times n}$ as a $d\times n$ dimensional vector. However, in this case, it is difficult to fully reflect the symmetry of the target function in terms of parameter efficiency, as separate parameters are required to construct column-wise monomials (as described in Section \ref{sec4}) for each individual column in $\boldsymbol{X}$. Similarly, computing the sum of column-wise monomials (as described in Section \ref{sec4.4}) also requires individual parameters for each column in $\boldsymbol{X}$. As a result, the number of parameters needed to construct a neural network equivalent to $\mathrm{TF}_1$ (in Proposition \ref{prop2}) increases linearly with the number of input columns, whereas it remains constant in Transformers. Hence, our construction requires fewer parameters compared to conventional neural networks when $d \ll n$ holds.

\subsection{Discussion on Number of Parameters}
In this study, monomial column-symmetric polynomials were used to universally approximate column-symmetric polynomials using single-headed Transformers, whose number of parameters does not depend on the number of input columns. This coincides with previous work \cite{Kaj24}, which demonstrates that single-layer Transformers are universal approximators. Furthermore, the number of feed-forward layers required by the Transformer is approximately $2sL + 3s$, which is comparable to the depth of Transformers used in practice.

In addition, when the degree of the target function $s$ is small, particularly when $s\leq 3$, the width of the Transformer can be of practical size. This corresponds to the case where the inputs interact with only a very limited number of other elements. In addition, the approximation error decreases exponentially with respect to the number of layers, while it only decreases polynomially with respect to the width. Hence, our results demonstrate the efficiency of deep Transformers.

On the other hand, the width of $\mathrm{TF}$ is proportional to $(2d)^s$, which becomes excessively large as $d$ and $s$ increase. This issue arises because the number of monomial column-symmetric polynomials of degree $s$ or less within a single column increases exponentially with $s$. Reducing the number of parameters to a practical level is a critical direction for future work to better understand the representational power of Transformers.

\subsection{Discussion on Positional Encoding}
When applying Transformers to various tasks, it is common to use positional encodings, values that distinguish individual columns, as in \cite{Vas17}. The use of positional encoding allows for discussions on more general cases, where symmetry is present only among specific columns. A similar discussion has been conducted in the context of neural networks in \cite{Mar19}. Exploring the impact of positional encodings raises intriguing questions: the extent of parameters required for such constructions and the specific architectures of Transformers. We leave these questions for future work.

\section*{Acknowledgements}
\label{sec8}
We would like to express our gratitude to Prof. Naoto Shiraishi and Prof. Koji Fukushima for fruitful comments and discussions.
This study was supported by JSPS KAKENHI (24K02904), JST CREST (JPMJCR21D2), and JST FOREST (JPMJFR216I).


\appendix

\section{Basic Mathematical Properties}
Here, we provide basic lemmas which are used in this paper.
    
\begin{lem}
\label{lm1}
For any $n \in \mathbb{N}$, the following equation holds.
\begin{equation*}
    \begin{pmatrix}  n \\ 0 \end{pmatrix} + \begin{pmatrix}  n \\ 1 \end{pmatrix} + \dots + \begin{pmatrix}  n \\ n \end{pmatrix} = 2^n.
\end{equation*}
Hence, $\displaystyle \binom{n}{k} \leq 2^n$ holds for any $n,k\in \mathbb{N}$.
\end{lem}

Lemma \ref{lm1} easily follows from the binomial theorem, and from $\binom{n}{k} = 0$ when $n < k$. The following lemma is a well-known result of classic combinatorics.

\begin{lem}
\label{lm2}
    The number of sets of integers $(p_1, p_2, \dots, p_n)$ which satisfy 
    \begin{equation*}
        p_1 + p_2 + \dots + p_n = k,\quad p_1, p_2, \dots, p_n \geq 0
    \end{equation*}
    are $\displaystyle \begin{pmatrix} k+n-1 \\ n-1 \end{pmatrix}$.
\end{lem}

\begin{proof}
    Let $s_0 = 0,\ s_i = p_1 + \dots + p_i + i\ (i = 1, 2, \dots, n)$. Then, the number of sets of integers $(p_1, p_2, \dots, p_n)$ are equivalent to the number of sets of integers $s_1, s_2, \dots, s_{n-1}$ which satisfy 
    \begin{equation*}
        s_0 = 0,\quad s_n = k + n,\quad s_i > s_{i-1}\ (i = 1, 2, \dots, n).
    \end{equation*}
    As this is equivalent to choosing $n-1$ distinct integers from 1 to $k+n-1$, the desired count is $\displaystyle \begin{pmatrix} k+n-1 \\ n-1 \end{pmatrix}$.
\end{proof}
Note that when the condition in Lemma \ref{lm2} is altered to $p_1 + p_2 + \dots + p_n \leq k$, by introducing an additional variable $p_{n+1} = k - (p_1 + \cdots + p_n)$, the problem can be reformulated as finding the total number of non-negative integer solutions to $p_1 + p_2 + \cdots + p_{n+1} = k$, which is $\displaystyle \begin{pmatrix} k+n \\ n \end{pmatrix}$.

The next lemma demonstrates that monomial column-symmetric polynomials generate the set of column-symmetric polynomials.

\begin{lem}
\label{lm3}
    Any column-symmetric polynomial $f(\boldsymbol{X})$ can be written in a linear combination of monomial column-symmetric polynomials.
\end{lem}

\begin{proof}
    First, we compare the degree of polynomials $P_1(\boldsymbol{X}) = \boldsymbol{x}_1^{\boldsymbol{u}_1} \dots \boldsymbol{x}_l^{\boldsymbol{u}_l}$ and $P_2(\boldsymbol{X}) = \boldsymbol{x}_1^{\boldsymbol{v}_1} \dots \boldsymbol{x}_l^{\boldsymbol{v}_l}$. If $\boldsymbol{u}_1 > \boldsymbol{v}_1$, we regard $P_1$ has a higher degree than $P_2$, and vice versa. If $\boldsymbol{u}_1 = \boldsymbol{v}_1$, we compare $\boldsymbol{u}_2$ and $\boldsymbol{v}_2$ and so on, similarly to the case of comparing vectors.
    
    Let $\boldsymbol{x}_1^{\boldsymbol{p}_1} \dots \boldsymbol{x}_l^{\boldsymbol{p}_l}$ be the monomial in  $f(\boldsymbol{X})$ with the highest degree. In this case, for any permutation $(\sigma(1), \sigma(2), \dots, \sigma(n))$ of $(1, 2, \dots, n)$, $\boldsymbol{x}_{\sigma(1)}^{\boldsymbol{p}_1} \dots \boldsymbol{x}_{\sigma(l)}^{\boldsymbol{p}_l}$ must have a lower degree than $\boldsymbol{x}_1^{\boldsymbol{p}_1} \dots \boldsymbol{x}_l^{\boldsymbol{p}_l}$, since $f(\boldsymbol{X})$ is column-symmetric and must contain these terms. Consider the polynomial 
    \begin{equation}
    \label{eq6}
        f(\boldsymbol{X}) - c \cdot \sum_{\sigma \in S_n} \boldsymbol{x}_{\sigma(1)}^{\boldsymbol{p}_1} \dots \boldsymbol{x}_{\sigma(l)}^{\boldsymbol{p}_l}, 
    \end{equation}
    where $c$ is set so that the coefficient of $\boldsymbol{x}_1^{\boldsymbol{p}_1} \dots \boldsymbol{x}_l^{\boldsymbol{p}_l}$ in \eqref{eq6} is equal to 0. In this case, \eqref{eq6} must have a degree lower than that of $f(\boldsymbol{X})$, because otherwise it would contradict the assumption that $\boldsymbol{x}_1^{\boldsymbol{p}_1} \dots \boldsymbol{x}_l^{\boldsymbol{p}_l}$ has the largest degree.
    
    By repeatedly applying this operation, eventually we arrive at a polynomial of degree 0, which is a constant. Thus, $f(\boldsymbol{X})$ must be able to be expressed as a linear combination of monomial column-symmetric polynomials.
\end{proof}

\section{Proof of Lemma \ref{lm7}}
\label{sec9.2}

\begin{lem}
\label{lm4}
    The equality 
    \begin{equation*}
        T_k(x) = T_1\left( 2^{k-1} \left(x - \frac{i}{2^{k-1}}\right) \right)
    \end{equation*}
    holds for any $x \in [0,1], k \geq 1$ and $\displaystyle \frac{i}{2^{k-1}} \leq x \leq \frac{i+1}{2^{k-1}}$, where $i \in \{0,1,\dots,2^{k-1}-1 \}$.
\end{lem}

\begin{proof}
    We prove the lemma by induction on $k$. The case for $k=1$ is trivial, since $2^{k-1} = 1$ implies $i=0$ for any $x\in [0,1]$. Next, we assume Lemma \ref{lm4} holds for $T_k(x)$. Since $\displaystyle \frac{i}{2^{k-1}} \leq x \leq \frac{i+1}{2^{k-1}}$, we can divide the discussion into two cases. 
    
    First, if $\displaystyle \frac{2i}{2^k} \leq x \leq \frac{2i+1}{2^k}$, then $\displaystyle 2^{k-1} \left(x - \frac{i}{2^{k-1}} \right) \leq \frac{1}{2}$ implies 
    \begin{equation*}
        T_k(x) =  T_1\left( 2^{k-1} \left(x - \frac{i}{2^{k-1}}\right) \right) = 2 \cdot 2^{k-1} \left(x - \frac{i}{2^{k-1}}\right) = 2^k \left(x - \frac{2i}{2^k}\right), 
    \end{equation*}
    resulting in $\displaystyle T_{k+1}(x) = T_1(T_k(x)) = T_1\left(2^k \left(x - \frac{2i}{2^k}\right) \right)$. 
    
    Second, if $\displaystyle \frac{2i+1}{2^k} \leq x \leq \frac{2(i+1)}{2^k}$, then $\displaystyle 2^{k-1} \left(x - \frac{i}{2^{k-1}} \right) \geq \frac{1}{2}$ implies 
    \begin{equation*}
        T_k(x) =  T_1\left( 2^{k-1} \left(x - \frac{i}{2^{k-1}}\right) \right) = 2 - 2\cdot 2^{k-1} \left(x - \frac{i}{2^{k-1}} \right) = 2 - 2^k \left(x - \frac{i}{2^{k-1}}\right).
    \end{equation*}
    It is obvious that $T_1(1-x') = T_1(x')$ holds for any $x' \in [0,1]$, so we obtain 
    \begin{equation*}
    \begin{split}
        \displaystyle T_{k+1}(x) = & T_1(T_k(x)) = T_1\left(2 - 2^k \left(x - \frac{i}{2^{k-1}}\right)\right) = T_1 \left(1 - \left(2 - 2^k \left(x - \frac{i}{2^{k-1}}\right) \right) \right) \\
        = & T_1 \left( 2^k \left(x - \frac{2i}{2^k}\right) - 1 \right) = T_1 \left( 2^k \left(x - \frac{2i + 1}{2^k}\right) \right).
    \end{split}
    \end{equation*}
    In either case, we get $\displaystyle T_{k+1}(x) = T_1 \left( 2^k \left(x - \frac{i'}{2^k}\right) \right)$, where $\displaystyle \frac{i'}{2^k} \leq x \leq \frac{i'+1}{2^k}$ and $i'\in \{ 0,1,\dots,2^k-1 \}$, which completes the induction.
\end{proof}

Now we can prove Lemma \ref{lm7}.

\begin{proof}
    We also prove this proposition by induction on $k$. The case for $k=1$ is easy because
    \begin{equation*}
        \tilde{f}_1(x) - x^2 = \left(x - \frac{1}{4} T_1(x)\right) - x^2 = 
        \begin{cases}
            \frac{1}{2} x - x^2 = -x\left( x - \frac{1}{2} \right) \quad \text{if}\ 0 \leq x \leq \frac{1}{2},\\
            (\frac{3}{2} x - \frac{1}{2}) - x^2 = -(x - \frac{1}{2})(x - 1) \quad \text{if}\ \frac{1}{2}\leq x \leq 1.
        \end{cases}
    \end{equation*}
    If Lemma \ref{lm7} holds for $\tilde{f}_k(x)$, the assumption and Lemma \ref{lm4} imply 
    \begin{equation*}
    \begin{split}
        \tilde{f}_{k+1}(x) - x^2 = & \tilde{f}_k (x) - x^2 + \frac{1}{4^{k+1}} T_{k+1}(x) \\
        = & - \left( x - \frac{i}{2^k} \right) \left( x - \frac{i+1}{2^k} \right) + \frac{1}{4^{k+1}} T_1\left( 2^k \left(x - \frac{i}{2^k}\right) \right).
    \end{split}
    \end{equation*}
    Hence, if $\displaystyle \frac{2i}{2^{k+1}} \leq x \leq \frac{2i+1}{2^{k+1}}$, then $\displaystyle 2^k \left(x - \frac{i}{2^k}\right) \leq \frac{1}{2}$ implies 
    \begin{equation*}
    \begin{split}
        & - \left( x - \frac{i}{2^k} \right) \left( x - \frac{i+1}{2^k} \right) + \frac{1}{4^{k+1}} T_1\left( 2^k \left(x - \frac{i}{2^k}\right) \right) \\
        = & - \left( x - \frac{i}{2^k} \right) \left( x - \frac{i+1}{2^k} \right) + \frac{1}{2^{k+1}} \left(x - \frac{i}{2^k}\right) = - \left( x - \frac{i}{2^k} \right) \left( x - \frac{2i+1}{2^{k+1}} \right).
    \end{split}
    \end{equation*}
    On the other hand, if $\displaystyle \frac{2i+1}{2^{k+1}} \leq x \leq \frac{2(i+1)}{2^{k+1}}$, then $\displaystyle 2^k \left(x - \frac{i}{2^k}\right) \geq \frac{1}{2}$ implies 
    \begin{equation*}
    \begin{split}
         & - \left( x - \frac{i}{2^k} \right) \left( x - \frac{i+1}{2^k} \right) + \frac{1}{4^{k+1}} T_1\left( 2^k \left(x - \frac{i}{2^k}\right) \right)\\
        = & - \left( x - \frac{i}{2^k} \right) \left( x - \frac{i+1}{2^k} \right) + \frac{1}{4^{k+1}} \left( 2 - 2^{k+1} \left(x - \frac{i}{2^k}\right) \right)\\
        = & - \left( x - \frac{i}{2^k} \right)  \left( x - \frac{i+1}{2^k} \right) + \frac{1}{2^{k+1}} \left( x - \frac{i+1}{2^{k+1}} \right) = - \left( x - \frac{2i+1}{2^{k+1}} \right) \left( x - \frac{i+1}{2^k} \right)
    \end{split}
    \end{equation*}
    For either case, $\displaystyle \tilde{f}_{k+1}(x) - x^2 = - \left( x - \frac{i'}{2^{k+1}} \right) \left( x - \frac{i'+1}{2^{k+1}} \right)$ holds for $\displaystyle \frac{i'}{2^{k+1}} \leq x \leq \frac{i'+1}{2^{k+1}}$, where $i'\in \{ 0,1,\dots,2^{k+1}-1 \}$, completing the induction. The latter statement of Lemma \ref{lm7} easily follows by completing the square.
\end{proof}

\section{Proof of Lemma \ref{lm5}}
\label{sec9.3}
Assume that the ReLU FNN with width $N$ and depth $L$ can be written as
\begin{equation*}
\tilde{\boldsymbol{x}}_{i+1} = 
\begin{cases}
    \mathrm{ReLU}(\boldsymbol{\tilde{W}}_i \tilde{\boldsymbol{x}}_i + \boldsymbol{b}_i) \quad \text{if } i < L, \\
    \boldsymbol{\tilde{W}}_i \tilde{\boldsymbol{x}}_i + \boldsymbol{b}_i \quad \text{if } i = L,
\end{cases}
\end{equation*}
\begin{equation*}
    (i = 0, 1, \dots, L,\ \tilde{\boldsymbol{x}}_i \in \mathbb{R}^{d_i},\ \boldsymbol{\tilde{W}}_i \in \mathbb{R}^{d_i\times d_{i+1}},\ d_{L+1} = 1)
\end{equation*}
where $\tilde{\boldsymbol{x}}_0 \in \mathbb{R}^{d_0}$ and $\tilde{x}_{L+1} \in \mathbb{R}$ are the inputs and the output respectively. Without loss of generality, we can assume that the dimension of $\tilde{\boldsymbol{x}}_0, \dots, \tilde{\boldsymbol{x}}_L$ are equal to $N$ by adding 0s to each bottom. Let $\tilde{\boldsymbol{\ell}}_0 = (\tilde{\boldsymbol{x}}_0^\top, \boldsymbol{0}_N^\top)^\top \in \mathbb{R}^{2N\times 1}$, and define
\begin{equation*}
    \tilde{\boldsymbol{\ell}}_{2i+1} = \boldsymbol{W}_{2i} \cdot \mathrm{ReLU} \left(
    \begin{pmatrix}
        -\boldsymbol{I}_N & \boldsymbol{O}_N\\
        \tilde{\boldsymbol{W}}_{2i} & \boldsymbol{O}_N
    \end{pmatrix}
    \tilde{\boldsymbol{\ell}}_{2i} + 
    \begin{pmatrix}
        \boldsymbol{0}_N \\ \boldsymbol{b}_{2i}
    \end{pmatrix}
    \right),
\end{equation*}
\begin{equation*}
    \tilde{\boldsymbol{\ell}}_{2i+2} = 
    \boldsymbol{W}_{2i+1} \cdot \mathrm{ReLU} \left(
    \begin{pmatrix}
        \boldsymbol{O}_N & \tilde{\boldsymbol{W}}_{2i+1}\\
        \boldsymbol{O}_N & -\boldsymbol{I}_N
    \end{pmatrix}
    \tilde{\boldsymbol{\ell}}_{2i+1} + 
    \begin{pmatrix}
        \boldsymbol{b}_{2i} \\ \boldsymbol{0}_N
    \end{pmatrix}
    \right),
\end{equation*}
\begin{equation*}
    \mathrm{where}\ \boldsymbol{W}_i = 
    \begin{cases}
        \begin{pmatrix}
            \boldsymbol{W}_L & \boldsymbol{W}_L\\
            \boldsymbol{O}_N & \boldsymbol{O}_N
        \end{pmatrix}
        \quad \text{if}\ i = L - 1,\\
        \boldsymbol{I}_{2N}\quad \text{otherwise}.
    \end{cases}
\end{equation*}
It is easy to check that the top $N$ elements of $\tilde{\boldsymbol{\ell}}_{L}$ are equal to $\tilde{\boldsymbol{x}}_{L+1}$, and constructing $\tilde{l}_i$ for $i = 1,2,\dots,L$ can be done by the feed-forward layers of Transformers.

\bibliography{sample}
\bibliographystyle{plainnat}

\end{document}